\documentclass[11pt]{article}

\usepackage[margin=1in]{geometry}
\usepackage{lmodern}
\usepackage[T1]{fontenc}
\usepackage[utf8]{inputenc}
\usepackage{microtype}
\usepackage[square,numbers]{natbib}

\usepackage{amsmath, amsthm, amssymb}
\newtheorem{theorem}{Theorem}
\newtheorem{definition}{Definition}
\newtheorem{corollary}{Corollary}
\usepackage{algorithm}
\usepackage{algpseudocode}
\usepackage{array}
\usepackage{pgfplots}
\usepackage{float}

\usepackage{tikz}
\usepackage{pgfplots}
\pgfplotsset{compat=1.18}
\pgfplotsset{/pgfplots/error bars/error bar style={very thick, black}}
\usepgfplotslibrary{groupplots}
\usepgfplotslibrary{statistics}
\usetikzlibrary{matrix, patterns, positioning, arrows.meta, fit}

\usepackage{hyperref}
\usepackage{cleveref}
\crefname{algorithm}{algorithm}{algorithms}
\Crefname{algorithm}{Algorithm}{Algorithms}

\usepackage[table, dvipsnames]{xcolor}
\usepackage{soul}
\usepackage{listings}
\definecolor{tabblue}{HTML}{1C77B4}
\definecolor{taborange}{HTML}{FF7F0B}
\definecolor{tabgreen}{HTML}{2AA02B}
\definecolor{tabred}{HTML}{D62729}
\definecolor{tabpurple}{HTML}{9567BC}
\definecolor{codebg}{rgb}{0.96,0.96,0.96}
\definecolor{codeframe}{rgb}{0.85,0.85,0.85}
\newcommand{\tokenhl}[2]{%
  \sethlcolor{OliveGreen!#1!white}%
  \hl{#2}%
}
\lstdefinestyle{customc}{
  backgroundcolor=\color{codebg},
  breaklines=true,
  breakindent=0pt,
  frame=single,
  rulecolor=\color{codeframe},
  resetmargins=true,
  language=Bash,
  showstringspaces=false,
  basicstyle=\footnotesize\ttfamily,
  keywordstyle=\text{},
  commentstyle=\itshape\color{purple},
  identifierstyle=\color{black},
  stringstyle=\color{black},
  numbers=none,
  mathescape=true,
  escapechar=@,
}
\usepackage{booktabs,tabularx}

\usepackage{fancyhdr}
\usepackage{balance}

\hypersetup{colorlinks,
            linkcolor=black,
            citecolor=black,
            urlcolor=OliveGreen,}

\providecommand{\keywords}[1]{\paragraph*{Keywords:} #1}

\newcommand{\BibTeX}{\rm B\kern-.05em{\sc i\kern-.025em b}\kern-.08em\TeX}

\title{Generating Fair Consensus Statements \\with Social Choice on Token-Level MDPs}

\author{Carter Blair\\
Harvard University\\
\texttt{carterblair@g.harvard.edu}
\and
Kate Larson\\
University of Waterloo\\
\texttt{kate.larson@uwaterloo.ca}
}

\date{October, 2025} 

\begin{document}
\maketitle

\begin{abstract}
 Current frameworks for aggregating free-form text-based opinions with large language models lack the inherent structure needed to provide meaningful fairness guarantees. To address this, we model the task as a multi-objective, token-level Markov Decision Process (MDP), where each objective corresponds to an agent's preference. Each agent’s token-level reward is induced by its policy (e.g., a personalized language model). Such policies implicitly define optimal Q-functions, thus enabling stepwise reward computation without an explicit value function~\cite{rafailov2024from}. This MDP formulation yields a formal structure that can be analyzed with tools from social choice theory. We first give a stochastic generation policy that is guaranteed to lie in the ex-ante core. It is derived from a distribution over complete statements that maximizes Nash welfare, extending core stability from cooperative game theory and voting to text generation. Second, for a single consensus statement, we target egalitarian welfare and use search within the MDP. Empirically, this search produces statements with improved worst-case agent alignment compared with baselines, including the Habermas Machine \citep{tessler2024ai}. Our code is available at \url{https://github.com/cartgr/Generating-Fair-Consensus-Statements-with-Social-Choice-on-Token-Level-MDPs}.
\end{abstract}



\section{Introduction}

Social choice theory has traditionally studied how to aggregate preferences over predefined alternatives. Large language models (LLMs) make it possible to aggregate free-form verbal opinions into collective textual outputs, which removes the constraint of a fixed agenda. Importantly, this flexibility can reduce the agenda-setting power of organizers, since participants need not choose from a preset list. However, ensuring provable fairness is difficult: complex training and design choices yield an opaque, irregular LLM structure, which in turn complicates the formalization of specific fairness criteria during generation. For this reason, previous approaches have treated the generation process as a black box, applying various fairness measures post-hoc. For instance, in the Habermas Machine \citep{tessler2024ai}, statements are first generated, and fairness is then pursued through a voting procedure applied to these statements, which were not themselves generated with an explicitly fair mechanism. Similarly, the Generative Social Choice method \citep{fish2024generative} prompts an LLM to maximize a given objective and assumes that the response truly maximizes the objective. These strategies, though aimed at fairness in free-form opinion aggregation, can a cede a new form of agenda control to the LLM. When the model is asked to produce text that satisfies a broad fairness objective, its interpretation and implementation of that objective, the trade-offs it chooses, and the aspects of opinions it elevates remain opaque. This effectively allows the LLM to shape the solution space. As such, these methods risk overlooking biases embedded within the generation process itself, which are known to exist~\citep{feng-etal-2023-pretraining}.

We address this gap by modeling consensus statement generation as a token-level Markov Decision Process (MDP). Each agent $i$'s viewpoint is represented by a policy $\pi_i$, which assigns likelihoods $\pi_i(s,a)$ to token choices given the current prefix $s$. Following~\citet{rafailov2024from}, who show that policies implicitly define optimal Q-functions, our agent policies, $\pi_i$, determine rewards $r_i(s,a)$ (e.g., $r_i^{\log}(s,a) = \beta \log \pi_i(s,a)$) at each generation step. A primary advantage of this reward formulation is that it avoids personalized value functions, which are known to be challenging to train and apply effectively~\citep{guo2025deepseek}. This MDP structure provides a formal basis for integrating fairness principles directly into the construction of the consensus statement.

Within this MDP framework, we develop two approaches that leverage existing notions of fairness from social choice theory, namely the \emph{ex-ante core} and \emph{egalitarian welfare (EW)}, which we argue are compelling notions of fairness in the context of consensus statement generation. First, to achieve an outcome in the ex-ante core, we propose a stochastic generation policy $\Pi^*$. This policy is derived by optimizing a distribution over complete statements to maximize proportional fairness (Nash Welfare), a process known to yield core membership. For consensus generation, the core is a highly desirable stability concept: a lottery in the core ensures that no coalition of agents could unilaterally deviate and achieve an alternative lottery that all its members prefer, given their proportional influence, implying agreement, nay consensus, over the randomized outcome. Second, when a single consensus statement is desired, we target the maximization of EW, seeking the best outcome for the least satisfied agent, which aligns with the idea that a consensus statement should be agreeable to all parties. We introduce constructive search algorithms (finite lookahead and beam search) that optimize this EW objective directly within the MDP. This offers a transparent and analyzable generation mechanism distinct from methods reliant on high-level prompting or post-hoc voting.

Our main contributions are:
\begin{enumerate}
\item A formal token-level MDP framework for fair consensus generation where agent rewards are derived from their language model policies.
\item A method to derive a stochastic generation policy that is provably in the ex-ante core, ensuring proportional fairness and stability.
\item The development and empirical validation of search algorithms that, by optimizing egalitarian welfare within the MDP, generate single consensus statements with improved worst-case agent alignment compared to methods that do not leverage this token-level structure or search.
\end{enumerate}

Through these contributions, we hope to establish a new direction for methods seeking to generate consensus statements from open-ended verbal opinions with provable fairness guarantees.


\section{Related Work}
\label{sec:related_work}

\paragraph{Generative Social Choice.}
This field applies social choice principles to open-ended generation, such as creating text from diverse opinions \citep{fish2024generative, tessler2024ai, revel2025ai, boehmer2025generative}. Unlike methods that aggregate preferences over predefined alternatives, Generative Social Choice (GSC) generates the alternatives themselves. For example, \citet{tessler2024ai} employ iterative critiques and voting on complete statements for consensus, while \citet{fish2024generative} prompt LLMs to directly maximize egalitarian welfare within a larger framework aimed at a form of proportional representation. Our work differs by embedding fairness into the token-by-token construction of a consensus statement via a multi-objective MDP, treating each token selection as a public decision \citep{conitzer2017fair}. This provides a more granular and verifiable mechanism than post-hoc evaluations or high-level prompting.

\paragraph{Randomized Social Choice.}
We also draw from randomized social choice, which studies lotteries over outcomes. Our stochastic generation policy, which maximizes Nash Welfare for proportional fairness, connects to this area. Maximizing Nash Welfare is known to yield outcomes in the 1-core \citep{aziz2019fair, fain2018fair, ebadian2024optimized}. We extend these findings to the sequential decision-making context of this paper.

\paragraph{Guided Decoding.}
Guided decoding techniques steer LLM generation towards desired attributes at inference time, often using search algorithms. Methods like PPO-MCTS \citep{liu2024dont} and VAS \citep{han2024value} use a value network to guide generation, while MOD \citep{shi2024decoding} or COLLAB \citep{chakraborty2025collab} combine or switch policies. Our approach also uses search but derives token-level rewards from agent policies. Further, we explicitly frame generation as planning in a multi-objective MDP to optimize social choice objectives (Proportional Fairness, Egalitarian Welfare), rather than relying on a single pre-trained value model or heuristic model combinations.


\section{Problem Setup \& Preliminaries}
\label{sec:problem_setup}

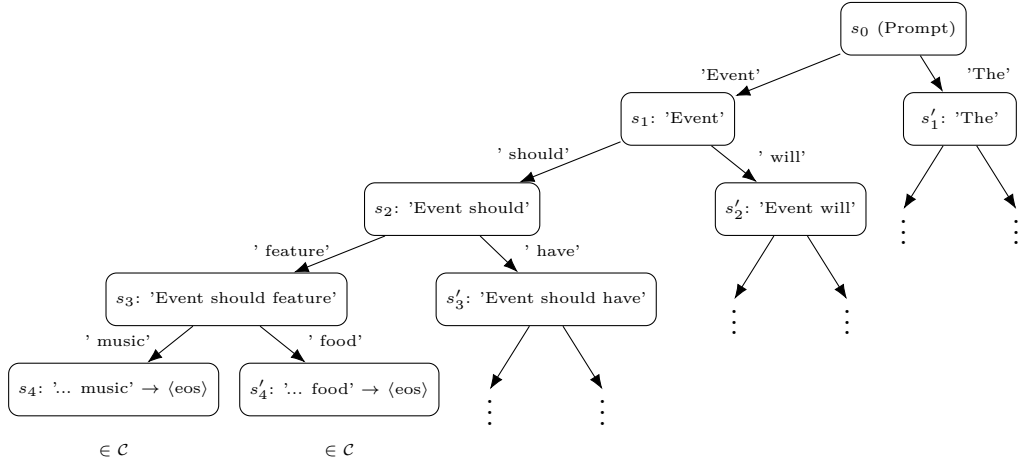
\begin{figure*}[t]
\centering
\begin{tikzpicture}[
    level distance=1.2cm,
    level 1/.style={sibling distance=10cm},    
    level 2/.style={sibling distance=6cm},    
    level 3/.style={sibling distance=6cm}, 
    level 4/.style={sibling distance=3cm}, 
    edge from parent/.style={draw, -{Latex[length=2mm, width=1.5mm]}},
    every node/.style={rectangle, draw, rounded corners, font=\tiny, align=center, minimum height=0.7cm, minimum width=1.5cm},
    leaf/.style={coordinate},
    vdots node/.style={draw=none, font=\normalsize},
    leaf label/.style={draw=none}
]
\node (root) {$s_0$ (Prompt)}
    child[sibling distance=6cm] { node (s1) {$s_1$: 'Event'} 
        child { node (s2) {$s_2$: 'Event should'}
            child { node (s3) {$s_3$: 'Event should feature'} 
                child { node[level 4] (leaf1) {$s_4$: '... music' $\rightarrow$ $\langle \text{eos}\rangle$} 
                    edge from parent node[left, draw=none, pos=0.4] {\tiny{' music'}}
                }
                child { node[level 4] (leaf2) {$s'_4$: '... food' $\rightarrow$ $\langle \text{eos}\rangle$}
                    edge from parent node[right, draw=none, pos=0.4] {\tiny{' food'}} 
                }
                edge from parent node[left, draw=none, pos=0.4] {\tiny{' feature'}}
            }
            child[sibling distance=2.5cm] { node (s3alt) {$s'_3$: 'Event should have'} 
                child[sibling distance=1.5cm] { coordinate (leaf3_coord) edge from parent node[left, draw=none, pos=0.6] {} }
                child[sibling distance=1.5cm] { coordinate (leaf4_coord) edge from parent node[right, draw=none, pos=0.6] {} } 
                edge from parent node[right, draw=none, pos=0.4] {\tiny{' have'}}
            } 
            edge from parent node[left, draw=none, pos=0.3] {\tiny{' should'}} 
        } 
        child[sibling distance=3cm] { node (s2alt) {$s'_2$: 'Event will'} 
             child[sibling distance=1.5cm] { coordinate (leaf5_coord) edge from parent node[left, draw=none, pos=0.6] {} } 
             child[sibling distance=1.5cm] { coordinate (leaf5b_coord) edge from parent node[right, draw=none, pos=0.6] {} } 
             edge from parent node[right, draw=none, pos=0.3] {\tiny{' will'}}
        } 
        edge from parent node[left, draw=none, pos=0.5] {\tiny{'Event'}} 
    } 
    child[sibling distance=1.5cm] { node (s1alt) {$s'_1$: 'The'}
        child[sibling distance=1.5cm] { coordinate (leaf6_coord) edge from parent node[left, draw=none, pos=0.6] {} } 
        child[sibling distance=1.5cm] { coordinate (leaf7_coord) edge from parent node[right, draw=none, pos=0.5] {} } 
        edge from parent node[right, draw=none, pos=0.5] {\tiny{'The'}}
    };

\node[vdots node] at ($(leaf3_coord) + (0,-0.2cm)$) {$\vdots$};
\node[vdots node] at ($(leaf4_coord) + (0,-0.2cm)$) {$\vdots$};
\node[vdots node] at ($(leaf5_coord) + (0,-0.2cm)$) {$\vdots$};
\node[vdots node] at ($(leaf5b_coord) + (0,-0.2cm)$) {$\vdots$};
\node[vdots node] at ($(leaf6_coord) + (0,-0.2cm)$) {$\vdots$};
\node[vdots node] at ($(leaf7_coord) + (0,-0.2cm)$) {$\vdots$};

\node[draw=none, below=0.1cm of leaf1] {$\in \mathcal{C}$};
\node[draw=none, below=0.1cm of leaf2] {$\in \mathcal{C}$};

\end{tikzpicture}
\caption{Illustration of the token-level generation tree. Edges represent actions (chosen tokens).}
\label{fig:token_tree_example_final_compact}
\end{figure*}

We consider a setting with a finite set of agents $N = \{1,2,\ldots,n\}$, each with a distinct opinion on a specific \textbf{Issue}. The goal is to generate a consensus text statement reflecting these perspectives fairly. The inputs to the process include descriptions of the issue, a policy representing each agent (which can be derived from free-form text expressing their opinion), and a \textbf{reference policy} (e.g., a base language model) that is used to propose tokens in the consensus statement.

\paragraph{Agent Policies.}
Each agent $i \in N$ is represented by a \emph{policy} $\pi_i$, which assigns a likelihood $\pi_i(s,a) \in [0, 1]$ to each action $a$ given the state $s$. Intuitively, $\pi_i(s,a)$ reflects how closely an action aligns with agent $i$'s preference at state $s$. This policy could be an LLM fine-tuned (ideally with DPO \citep{rafailov2023direct}) or base policy prompted with agent $i$'s viewpoint.

\paragraph{Token-Level MDP}
We model text generation as a deterministic, discrete-time Markov Decision Process defined by the tuple $(S, A, T, \mathbf{R})$. Here, $S$ is the state space of partial text sequences (prefixes), including initial $s_0$ and terminal states. $A$ is the action space consisting of the token vocabulary plus a special end-of-sequence token $\langle \text{eos}\rangle$. $T$ is the deterministic transition function where $T(s,a) = s \| a$ appends the chosen token; selecting $a = \langle \text{eos}\rangle$ leads to a terminal state representing a completed statement $X$. Finally, $\mathbf{R}$ represents the agent-specific reward functions. We define two types of rewards based on agent policies, serving different analytical purposes:
\begin{enumerate}
    \item \textbf{Log-Likelihood Reward:} $r_i^{\log}(s,a) = \beta \log \pi_i(s,a)$. This formulation aligns with implicit rewards in preference learning \citep{rafailov2024from}, where $\beta > 0$ is a scaling factor. This is non-positive and is suitable for additive utility accumulation along a path.
    \item \textbf{Likelihood Reward:} $r_i^{\text{prob}}(s,a) = \pi_i(s,a)$. This reward uses the direct probability, ensuring non-negativity ($r_i^{\text{prob}} \ge 0$), which is needed for social welfare functions involving products or ratios, such as Nash Welfare.
\end{enumerate}
We denote by $\mathcal{C}$ the set of all possible complete paths (sequences ending in $\langle \text{eos}\rangle$) from $s_0$.

In practice, we assume that at each non-terminal state $s$, we only consider a finite set of $B$ possible next tokens $A_B(s) \subseteq A$. This set could be the model's vocabulary or a subset chosen by a base language model giving us the $B$ most likely next tokens. With this finite branching factor $B$ and a maximum sequence length $L_{\max}$, the set $\mathcal{C}$ of complete paths is finite (bounded by $B^{L_{\max}}$).

This sequential token selection process naturally defines a tree structure rooted at $s_0$. Each edge represents choosing a token from $A_B(s_t)$, and each node represents a partial sequence $s_t$. To make this concrete, we provide an example in \autoref{fig:token_tree_example_final_compact}.

\paragraph{Agent Utilities.}
Given a completed sequence $X = (a_1, \dots, a_\ell = \langle \text{eos}\rangle)$ corresponding to states $(s_0, s_1, \dots, s_\ell)$, we define two corresponding utility functions for each agent $i$, derived from the respective reward types:
\begin{enumerate}
    \item \textbf{Additive Log-Utility:} Primarily used for evaluating single paths based on cumulative log-likelihood.
    \begin{align*}
        U_i^{\log}(X) &= \sum_{t=1}^{\ell} r_i^{\log}(s_{t-1}, a_t) = \sum_{t=1}^{\ell} \beta \log \pi_i(s_{t-1}, a_t) 
                       = \beta \log \left( \prod_{t=1}^{\ell} \pi_i(s_{t-1}, a_t) \right)
    \end{align*}
    \item \textbf{Multiplicative Probability Utility:} Primarily used for evaluating distributions via expected utility, forming the basis for Nash Welfare and Proportional Fairness calculations.
    \[
    U_i^{\text{prob}}(X) = \prod_{t=1}^{\ell} r_i^{\text{prob}}(s_{t-1}, a_t) = \prod_{t=1}^{\ell} \pi_i(s_{t-1}, a_t) = P_i(X)
    \]
    This represents the joint probability of sequence $X$ under agent $i$'s policy.
\end{enumerate}
These are related by $U_i^{\log}(X) = \beta \log U_i^{\text{prob}}(X)$, with $U_i^{\text{prob}}(X) > 0$.

\subsection{Deterministic vs. Stochastic Policies}
We consider two types of policies for generating fair consensus statements in our token-level MDP:
\begin{enumerate}
    \item A deterministic policy \(\mu(s)\) that gives us a single path $X \in \mathcal{C}$. To evaluate deterministic policies, we adopt additive log-utilities $U_i^{\log}(X)$.
    \item A stochastic policy \(\pi(a|s)\) that gives us a distribution $p \in \Delta(\mathcal{C})$ over paths (a lottery). When assessing this type of outcome, we consider the expected probability-based utility
    \[
    \mathbb{E}_{X \sim p}[U_i^{\text{prob}}(X)] = \sum_{X \in \mathcal{C}} p(X) U_i^{\text{prob}}(X).
    \]
    When it is clear from the context, we adopt the shorthand of \(U_i^{\text{prob}}(p)\) to refer to the expected utility of agent \(i\) for a given distribution over paths. This expected utility $U_i^{\text{prob}}(p)$ is guaranteed to be non-negative, which is required for downstream fairness measures.
\end{enumerate}

\paragraph{The Fairness of a Deterministic Policy.}
We assess the fairness of a path $X \in \mathcal{C}$ given by a deterministic policy through its egalitarian welfare (EW), drawing from Rawls' maximin principle~\citep{rawls1971egalitarian}. This is defined using the additive log-utilities $U_i^{\log}(X)$:
\begin{equation}
\text{EW}^{\log}(X) = \min_{i \in N} U_i^{\log}(X) = \min_{i \in N} \sum_{t=1}^{\ell} \beta \log \pi_i(s_{t-1}, a_t).
\label{eq:egal} 
\end{equation}
Maximizing $\text{EW}^{\log}(X)$ means finding the path whose cumulative log-likelihood is highest for the agent who prefers it least. By maximizing the minimum utility, the egalitarian objective promotes broadly acceptable outcomes, which supports the requirement that a consensus statement should be agreeable to all.

\paragraph{The Fairness of a Stochastic Policy.}
To analyze the fairness of a distribution over paths given by our stochastic policy, we use the non-negative expected utilities $U_i^{\text{prob}}(p)$. Our goal is to find a stochastic policy that gives us a distribution \(p \in \Delta(\mathcal{C)}\) that is in the ex-ante core~\citep{shapley1971cores}. Following \citet{fain2018fair}, we define the ex-ante core as follows. 

\begin{definition}[Ex-Ante Core]
A distribution $p \in \Delta(\mathcal{C})$ is in the \textit{ex-ante core} if there is no coalition $S \subseteq N$ and alternative distribution $p'$ such that
\[
\frac{|S|}{|N|} \cdot U_i^{\text{prob}}(p') \geq U_i^{\text{prob}}(p), \quad \forall i \in S,
\]
with strict inequality for at least one agent $i \in S$. 
\end{definition}

Intuitively, if a distribution over paths is in the ex-ante core, there is no possible coalition of agents that could use their proportional share of probability to create a distribution that has strictly higher expected utility for at least one agent and not less expected utility for all other agents (i.e., a Pareto improvement)~\citep{aziz2019fair}. This resistance to coalitional deviation is desirable for consensus statements, as a distribution in the ex-ante core represents an outcome that all parties have implicitly agreed to (since they cannot use their proportional share of utility to do something better).


\section{Defining a Stochastic Policy in the Core}
\label{sec:stochastic_policies}

Having established the token-level MDP and fairness criteria, we now turn to defining a \textit{stochastic generation policy} $\pi^*$ that produces a distribution $p^*$ over complete consensus statements $\mathcal{C}$ (i.e., a lottery) that is in the ex-ante core. 

To do so, we first generate the tree of possible token sequences with branching factor $B$ and maximum length $L$. We then, find the distribution over paths \(p^*\) that maximizes Nash welfare (NW). The NW optimal distribution is

\[
p^*= \max_{p\in\Delta(\mathcal{C})}\text{NW}(p) = \max_{p\in\Delta(\mathcal{C})} \prod_{i=1}^n U_i^{\text{prob}}(p). 
\]

 We can find \(p^*\) by optimizing over the probability simplex $\Delta(\mathcal{C})=\{p\in\mathbb{R}^{|\mathcal{C}|}_{\ge 0}:\sum_{X\in\mathcal{C}}p(X)=1\}$. Indexing the leaves as $\mathcal{C}=\{X_1,\dots,X_m\}$ and defining $u_i\in\mathbb{R}^m_{\ge 0}$ by $u_i(j)=U_i^{\text{prob}}(X_j)$, we get that each agent’s expected utility under distribution $p$ is
\[
U_i^{\text{prob}}(p)=\sum_{j=1}^m u_i(j)\,p_j \;=\; u_i^\top p,
\]
so the Nash welfare program is
\[
p^*=\max_{p\in\Delta(\mathcal{C})}\ \sum_{i=1}^n \log\!\big(u_i^\top p\big).
\]
Note, each term $\log(u_i^\top p)$ is concave (log is concave and increasing; $u_i^\top p$ is affine), and the simplex is convex; hence this is a convex program with polynomial-time algorithms and strong duality under a mild positivity condition (i.e., $u_i^\top p>0$ for all $i$) \citep{boyd2004convex}. 

Moreover, we know that any maximizer of $\prod_i U_i^{\text{prob}}(p)$ lies in the core \citep{fain2018fair,aziz2019fair,ebadian2024optimized}. So, we can work backwards from this distribution to find a stochastic policy in the core.

Specifically, given the target distribution $p^* \in \Delta(\mathcal{C})$ that maximizes Nash welfare, we derive the stochastic policy $\pi^*$ that generates this distribution. 
To define $\pi^*$, we introduce some notation. For any state (prefix) $s$ in the token-level MDP:
\begin{itemize}
    \item Let $\mathcal{C}(s) \subseteq \mathcal{C}$ be the set of all complete paths (leaves) that pass through state $s$.
    \item Let $\mathcal{C}(s, a) \subseteq \mathcal{C}(s)$ be the subset of paths in $\mathcal{C}(s)$ where the next action taken from state $s$ is $a$. Note that $\mathcal{C}(s, a) = \mathcal{C}(s\|a)$, where $s\|a$ is the state reached after taking choosing token $a$.
    \item For any subset of leaves $L \subseteq \mathcal{C}$, let $P^*(L) = \sum_{X \in L} p^*(X)$ be the total probability mass assigned by the NW optimal distribution $p^*$ to the leaves in $L$.
\end{itemize}
Note that $P^*(\mathcal{C}(s_0)) = P^*(\mathcal{C}) = 1$.

With this, we can define the policy $\pi^*$ at any given state $s$.

\begin{definition}[Stochastic Policy Induced by $p^*$]
\label{def:induced_policy}
Let $p^*$ be a distribution over the leaf nodes $\mathcal{C}$. The induced stochastic policy $\pi^*$ at a non-terminal state $s$ assigns the probability of taking the next action (token) \(a\) as:
\begin{align}
\pi^*(a|s) = \begin{cases}
\frac{P^*(\mathcal{C}(s, a))}{P^*(\mathcal{C}(s))} & \text{if } P^*(\mathcal{C}(s)) > 0 \\
0 & \text{if } P^*(\mathcal{C}(s)) = 0
\end{cases}
\end{align}

\end{definition}
This represents the conditional probability, according to the NW optimal distribution $p^*$, of selecting token $a$ next, given that the generation process has reached state $s$. If state $s$ has zero probability of being reached under $p^*$ (i.e., $P^*(\mathcal{C}(s)) = 0$), then the probability of taking any action from $s$ is also zero. 

\autoref{alg:core_policy} summarizes the process of finding a policy in the ex-ante core. We will now briefly turn to the runtime of this algorithm. 

\paragraph{Runtime.}
We optimize the Nash–welfare program on $\Delta(\mathcal{C})$ (with $m=|\mathcal{C}|$ as above) and then induce $\pi^*$ via the conditional masses $P^*(\mathcal{C}(s))$ and $P^*(\mathcal{C}(s,a))$. Forming all agent–leaf utilities costs $O(nmL)$. We then minimize $-\sum_i \log(u_i^\top p)$ on the simplex with Frank–Wolfe: each iteration computes $\nabla F(p)=\sum_i u_i/(u_i^\top p)$ in $O(nm)$ time and uses a coordinate linear oracle; the method attains $\varepsilon$-suboptimality in $O(1/\varepsilon)$ iterations \citep{jaggi2013revisiting}. Hence the optimization time is $O(nm/\varepsilon)$. The conditional masses needed for $\pi^*$ are obtained by a single bottom-up pass over the generation tree, which is linear in the number of leaves when $B$ is fixed, i.e., $\Theta(m)$. Altogether, the total runtime is
\[
T_{\text{total}} \;=\; O\!\big(nmL \;+\; nm/\varepsilon \;+\; m\big),
\]
which is polynomial in $n$, $m$, and $1/\varepsilon$.

\subsection{Properties of the Stochastic Policy}
\label{ssec:properties_induced_policy}

We now establish that executing this policy $\pi^*$ from the initial state $s_0$ indeed generates the target distribution $p^*$.

\begin{theorem}[Equivalence of Policy-Induced Distribution and Target Lottery]
\label{thm:policy_equivalence}
Let $p_{\pi^*}$ be the distribution over $\mathcal{C}$ generated by executing the policy $\pi^*$ (defined in Definition~\ref{def:induced_policy}) from the initial state $s_0$. Then $p_{\pi^*} = p^*$.
\end{theorem}
The proof is presented in \autoref{app:proof}. This equivalence leads to the desired ex-ante fairness guarantee for the policy $\pi^*$.

\begin{corollary}[Core Membership of Stochastic Policy]
\label{cor:policy-core}
Let $p^*$ be a distribution over $\mathcal{C}$ that maximizes Nash Welfare (and is therefore in the ex-ante core). Let $\pi^*$ be the stochastic policy derived from $p^*$ according to Definition~\ref{def:induced_policy}. Then the distribution $p_{\pi^*}$ generated by executing $\pi^*$ is in the ex-ante core.
\end{corollary}

\begin{proof}
By Theorem~\ref{thm:policy_equivalence}, the distribution generated by policy $\pi^*$ is $p_{\pi^*} = p^*$. Since $p^*$ was chosen to maximize Nash Welfare over $\mathcal{C}$, it is in the core. Therefore, $p_{\pi^*}$ is also in the core.
\end{proof}

This simple result confirms that our procedure, which first finds the distribution maximizing Nash Welfare $p^*$ and then executes the derived policy $\pi^*$, yields a stochastic policy in the core. And thus, we have an ex-ante fair way to generate consensus statements.

\begin{algorithm}[t]
\caption{Compute core-stochastic policy $\pi^*$}
\label{alg:core_policy}
\begin{algorithmic}[1]
\Require Leaf set $\mathcal{C}=\{X_1,\dots,X_m\}$, agent utilities $U_i^{\text{prob}}(X_j)$
\State Build $u_i\in\mathbb{R}^m_{\ge 0}$ with $u_i(j)=U_i^{\text{prob}}(X_j)$
\State Solve $p^* \in \arg\max_{p\in\Delta(\mathcal{C})} \sum_{i=1}^n \log(u_i^\top p)$
\State For every state $s$ in the tree, cache $P^*(\mathcal{C}(s))$ and, for each enabled action $a$ at $s$, cache $P^*(\mathcal{C}(s,a))$.

\If{$P^*(\mathcal{C}(s))>0$}
  \State $\pi^*(a\mid s)\gets \dfrac{P^*(\mathcal{C}(s,a))}{P^*(\mathcal{C}(s))}$
\Else
  \State $\pi^*(a\mid s)\gets 0$
\EndIf
\State \Return $\pi^*$
\end{algorithmic}
\end{algorithm}

\subsection{Empirical Core Test}\label{sec:core_synth}

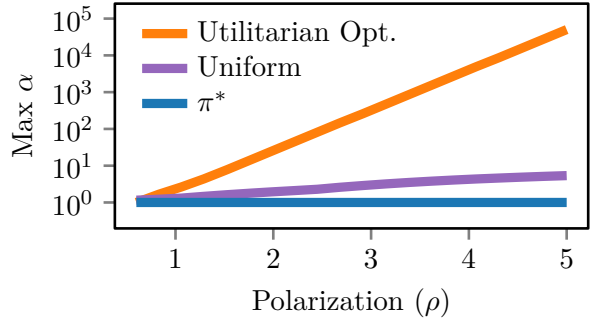
\begin{figure}
\centering
\begin{tikzpicture}
\begin{axis}[
  width=3.10in, height=1.80in,
  ymode=log, log basis y=10,
  enlarge x limits=0.05,
  enlarge y limits=0.15,
  xlabel={Polarization ($\rho$)},
  ylabel={Max $\alpha$},
  ytick={1,10,100,1000,10000,100000},
  yticklabels={$10^{0}$,$10^{1}$,$10^{2}$,$10^{3}$,$10^{4}$,$10^{5}$},
  tick pos=left, tick align=outside,
  axis line style={line width=1pt},
  tick style={line width=1pt},
  legend pos=north west,
  legend cell align={left},                 
  legend style={
    fill=none, draw=none,
    nodes={align=left, anchor=west}         
  },
  line join=round,
]

\addplot+[no marks, draw=taborange, line width=3.5pt]
coordinates {
(0.60,1.0000) (0.83,1.6770) (1.06,2.6392) (1.29,4.3734) (1.53,7.8455)
(1.76,14.0742) (1.99,25.2478) (2.22,45.2922) (2.45,81.2500) (2.68,145.7550)
(2.92,261.4711) (3.15,469.0550) (3.38,841.4416) (3.61,1509.4689)
(3.84,2707.8483) (4.07,4857.6308) (4.31,8714.1428) (4.54,15632.3707)
(4.77,28043.0352) (5.00,50306.6258)
};
\addlegendentry{Utilitarian Opt.}

\addplot+[no marks, draw=tabpurple, line width=3.5pt]
coordinates {
(0.60,1.1548) (0.83,1.2262) (1.06,1.3154) (1.29,1.4672) (1.53,1.6293)
(1.76,1.7875) (1.99,1.9481) (2.22,2.1088) (2.45,2.2869) (2.68,2.5893)
(2.92,2.8933) (3.15,3.1958) (3.38,3.4945) (3.61,3.7876) (3.84,4.0735)
(4.07,4.3515) (4.31,4.6208) (4.54,4.8812) (4.77,5.1327) (5.00,5.3751)
};
\addlegendentry{Uniform}

\addplot+[no marks, draw=tabblue, line width=3.5pt]
coordinates {
(0.60,1.0000) (0.83,1.0001) (1.06,1.0001) (1.29,1.0002) (1.53,1.0002)
(1.76,1.0002) (1.99,1.0003) (2.22,1.0003) (2.45,1.0003) (2.68,1.0003)
(2.92,1.0004) (3.15,1.0004) (3.38,1.0005) (3.61,1.0005) (3.84,1.0007)
(4.07,1.0007) (4.31,1.0008) (4.54,1.0008) (4.77,1.0009) (5.00,1.0009)
};
\addlegendentry{$\pi^{*}$}

\end{axis}
\end{tikzpicture}
\caption{Maximum coalition improvement vs.\ polarization.}
\label{fig:core_violation_plot}
\end{figure}

Here we empirically validate that the NW optimal \emph{policy} $\pi^*$ is in fact in the ex-ante core, and demonstrate that a uniform policy and a utilitarian optimal policy are not in the ex-ante core.

\paragraph{Environment} For demonstration we use a small $B$-ary token tree ($B{=}3$, depth $L{=}4$). Each agent $i$ has a next-token policy
\[
\pi_i(a\mid s)\;=\;\mathrm{softmax}_a\!\big(\rho\, w_i^\top(z+v_{t,a})\big),
\]
where $z$ is the running state embedding, $v_{t,a}$ is the token vector, and $w_i$ is the agent vector. The scalar \emph{polarization} $\rho\!\ge\!0$ controls how concentrated preferences are. When $\rho=0$, $\pi_i(a\mid s)=1/B$. For two actions $a,b$,
\[
\frac{\pi_i(a\mid s)}{\pi_i(b\mid s)} \;=\; \exp\!\big(\rho\, w_i^\top(v_{t,a}-v_{t,b})\big),
\]
so increasing $\rho$ multiplies odds exponentially and makes agents much more confident about their preferred action. 

To initialize the experiment we draw token vectors $\tilde v_{t,a}\sim\mathcal N(0,I_d)$ and agent vectors $\tilde w_i\sim\mathcal N(0,I_d)$, set $v_{t,a}=\tilde v_{t,a}/\|\tilde v_{t,a}\|_2$ and $w_i=\tilde w_i/\|\tilde w_i\|_2$, and define the state vector at prefix $s=(a_1,\ldots,a_{t-1})$ as $z(s)=\sum_{\tau=1}^{t-1} v_{\tau,a_\tau}$.

\paragraph{Baselines.}
We compare \(\pi^*\) to a \emph{uniform} policy that induces a uniform distribution over leaves and a \emph{utilitarian} policy that deterministically follows the single path maximizing $\sum_i U_{ij}$.

\paragraph{Blocking test.} For any policy $\pi$, let $p_\pi$ be its induced leaf distribution and $u_i(\pi)=U_i^\top p_\pi$. For a coalition $S$ of size $r$, give it a budget $r/n$ of probability mass and solve a small LP to find the largest factor $\alpha$ such that all $i\in S$ can secure at least $\alpha\,u_i(\pi)$. The \emph{maximum coalition improvement}
\[
\alpha^*(\pi)\;=\;\max_{S\neq\emptyset}\;\max_{p':\,\mathbf{1}^\top p'=r/n,\,p'\ge 0}\;
\min_{i\in S}\frac{U_i^\top p'}{u_i(\pi)}
\]

The value of \(\alpha^*(\pi)\) is $>1$ exactly when a coalition can block $\pi$. Figure~\ref{fig:core_violation_plot} plots $\alpha^*(\pi)$ versus polarization $\rho$ (log $y$-axis).

\paragraph{Takeaways.} \autoref{fig:core_violation_plot} shows the result. As $\rho$ increases, agents’ own policies become more concentrated. The Nash-welfare \emph{policy} stays at $\alpha^*(\pi^*)= 1$ across all $\rho$, consistent with the theoretical result. The utilitarian policy becomes highly blockable, and the uniform policy becomes steadily easier to block.

\subsection{Computational Tractability via Token Chunking}

The set of complete statements $\mathcal{C}$ grows as $|\mathcal{C}|\approx B^{L_{\max}}$ for branching factor $B$ and maximum length $L_{\max}$. Enumerating leaves and solving for $p^*$ over $\mathcal{C}$ could be infeasible for long sequences.

We therefore restrict attention to a coarser action space obtained by grouping tokens into macro-actions.

\begin{definition}[Token chunking]
A chunking scheme $\mathcal{K}$ partitions positions into contiguous blocks $\{k_1,\dots,k_m\}$, where each $k_j$ contains one or more tokens and the concatenation of chosen blocks forms a complete statement. Decoding now selects entire blocks. The corresponding feasible set of leaves is $\mathcal{C}_{\mathcal{K}}\subseteq \mathcal{C}$.
\end{definition}

With fixed chunk size $c$, the effective depth drops from $L$ to $\lceil L/c\rceil$. Running \autoref{alg:core_policy} on the chunked tree produces a lottery that lies in the ex-ante core \emph{with respect to} $\mathcal{C}_{\mathcal{K}}$; all coalitions are evaluated against this feasible set. This interpretation mirrors real elections: not every eligible candidate runs, yet we judge the outcome by comparing the candidates who did run. Here, chunking plays the role of eligibility: it narrows the set of feasible statements, and fairness is defined relative to that set.

However, this restriction does have a consequence. Because $\mathcal{C}_{\mathcal{K}}\subseteq \mathcal{C}$, the optimal Nash welfare over $\mathcal{C}_{\mathcal{K}}$ cannot exceed that over $\mathcal{C}$. Moreover, no uniform multiplicative guarantee is possible without further structure:

\begin{theorem}[No constant-factor approximation under chunking]\label{thm:no-const}
For any constant $R>1$, there exists a two-agent instance, a finite set of leaves $C$, and a chunked subset $K\subset C$ such that, writing
\[
\mathrm{NW}_Y(p)\;=\;\prod_{i=1}^{2}\Big(\sum_{x\in Y} u_i(x)\,p(x)\Big)\quad\text{for }Y\subseteq C,\;p\in\Delta(Y),
\]
and letting $p_Y^*\in\arg\max_{p\in\Delta(Y)} \mathrm{NW}_Y(p)$, we have
\[
\frac{\mathrm{NW}_C(p_C^{*})}{\mathrm{NW}_K(p_K^{*})}\;>\;R.
\]
Hence the approximation ratio of the chunked solution, measured against the unchunked optimum, is unbounded.
\end{theorem}

The proof appears in the Appendix. The theorem states a worst case that arises when chunking prunes precisely those statements that both agents value highly. 

Observe that we can think about this as an anytime algorithm: enlarging $\mathcal{C}_{\mathcal{K}}$ can only improve the Nash welfare, and the ex-ante core guarantee continues to hold relative to the current feasible set.


\section{Generating a Single Statement}
\label{sec:single_statement}

Although our stochastic policy $\pi^*$ achieves a highly desirable ex-ante fairness guarantee, many practical applications require selecting a single consensus statement. In this case, our objective shifts from finding a fair distribution to identifying the single path (i.e., statement) that represents consensus.

\subsection{Finding the Egalitarian Path}
\label{ssec:rawlsian_path}

Given the token tree with leaf nodes $\mathcal{C}$, we aim to find a deterministic policy \(\mu^*\) that produces a single path $X^* \in \mathcal{C}$ that maximizes egalitarian welfare as defined in Equation~\ref{eq:egal}. Due to the size of the token tree, exhaustive search for $X^*$ may be intractable. We propose approximate algorithms to find high-quality paths, including finite-lookahead search and beam search, which are detailed below.

\paragraph{Finite Lookahead Search. }The finite lookahead algorithm operates with a rolling horizon. At each step $t$, it explores all possible paths $P$ of length up to $d$ originating from $s_t$. For each such path $P$, the algorithm evaluates the egalitarian welfare of the sequence formed by concatenating the path generated so far ($X_{prefix}$) with $P$. It then chooses the first action $a^*$ of the path $P^*$ that maximizes this lookahead evaluation, transitions to state $s_{t+1} = T(s_t, a^*)$, and repeats the process. This $d$-step lookahead can mitigate the potential for hedging inherent in greedy search. When no single immediate token is agreeable (i.e., results in high egalitarian welfare), a greedy method might select less informative tokens that avoid commitment. In contrast, a lookahead can identify longer sequences that, despite potentially controversial initial steps, lead to states with higher overall welfare, perhaps by expressing a concept with suitable qualifications. The algorithm is shown in \autoref{app:search_algos}.

\paragraph{Beam Search.} Beam search is a heuristic search algorithm that balances greedy search and exhaustive exploration, and has been effective in sequence generation tasks like machine translation and text generation \citep{martin2009speech, rafailov2024from}. Instead of pursuing only the single best option (greedy search) or all options (exhaustive search), beam search maintains a fixed number of the most promising partial paths (hypotheses), $w$ (the beam width), at each depth $t$. At each step, it expands paths in the beam by generating potential successor tokens. These candidates are then evaluated using the egalitarian welfare objective function, and only the top $w$ scoring paths are retained for the next step. The algorithm returns the highest-scoring complete path found within the beam at the maximum length or upon reaching a terminal state. The algorithm is shown in \autoref{app:search_algos}.

\section{Experiments}
\label{sec:experiments}

We conduct experiments with three main objectives. First, we test whether agent policies derived from prompting language models show meaningful correlation with human preferences. Second, we examine whether these prompted policies can perform credit assignment to generate meaningful token-level rewards (a necessary prerequisite for our search algorithms). Third, we evaluate how well our finite lookahead and beam search methods perform in generating single consensus statements when compared to baseline approaches. 

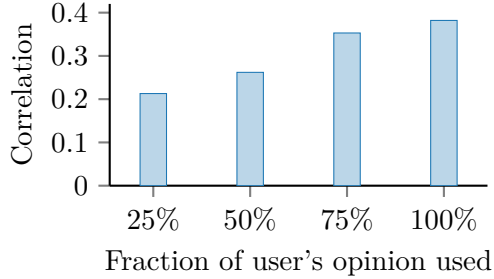
\begin{figure}
\centering
\begin{tikzpicture}
\begin{axis}[
  width=0.4\columnwidth, height=4.0cm,
  ybar,
  ymin=0,                      
  enlarge y limits={0.10,upper}, 
  enlarge x limits=0.15,
  axis x line*=bottom,         
  axis y line*=left,           
  axis line style={line width=0.8pt},
  tick style={line width=0.8pt},
  tick pos=left, tick align=outside,
  xlabel={Fraction of user's opinion used},
  ylabel={Correlation},
  xtick={25,50,75,100},
  xticklabels={25\%, 50\%, 75\%, 100\%},
  bar width=10pt
]
\addplot+[draw=tabblue, fill=tabblue!30]
  coordinates {(25,0.213) (50,0.262) (75,0.353) (100,0.382)};
\end{axis}
\end{tikzpicture}
\caption{Spearman's correlation between model likelihood and participants’ Likert ratings vs.\ the fraction of the user’s opinion provided.}
\label{fig:likelihood_likert_corr}
\end{figure}

\subsection{Opinion-length scaling experiment}
We test whether a chat LLM’s conditional likelihood of a policy statement, given a participant’s written opinion, predicts that participant’s 1–5 Likert rating, and whether supplying more of the participants opinion improves prediction.

\paragraph{Data.} We use the abortion dataset from \citet{fish2024generative}. This dataset contains data from 42 human participants who each provided their opinion on abortion in natural language. Each participant also rated five other candidate statements about abortion using a 1-5 Likert scale where higher values indicate higher agreement.
 
\paragraph{Method.}
We evaluate how well candidate statements capture participants' opinions using language model scoring. For each participant $u$ and statement $j$, we construct a paraphrase validation task using Meta-Llama-3.1-8B-Instruct-Turbo with temperature 1.0, computing log-probabilities without generation.

We structure the prompt as follows: the system instruction establishes the paraphrasing task, the user message provides the issue topic and the participant's original opinion, and we place the candidate statement as a pre-filled assistant response:
\begin{align}
&\texttt{[System] You paraphrase people's views.}\nonumber\\
&\texttt{[User] Topic: \{issue\}}\nonumber\\
&\texttt{\ \ \ \ \ \ \ Original opinion: \{opinion\}}\nonumber\\
&\texttt{[Assistant] Paraphrase: \{statement\}}\nonumber
\end{align}

Given the assistant response following the ``\texttt{Paraphrase:}'' tokenized as $a_1, \ldots, a_M$, we compute the length-normalized log-probability score:
\[
s_{u,j} = \frac{1}{M} \sum_{k=1}^{M} \log p_\theta(a_k \mid \text{context}, a_{1:k-1})
\]

Intuitively, this score measures how well the statement represents the participant's opinion. For validation, we compute Spearman's correlation between each participant's five model-assigned scores $\{s_{u,j}\}_{j=1}^5$ and their corresponding Likert ratings for the same statements.

\paragraph{Length manipulation.} We repeat the procedure after truncating each participants opinion to the last fraction $f\in\{1.00,0.75,0.50,0.25\}$ of its words.  

\paragraph{Results.}The correlation between the average likelihood of the statement according to the user prompted policy and the user's rating increases with the fraction of the user's opinion that is conditioned on (Fig.~\ref{fig:likelihood_likert_corr}). Thus, conditioning on more of the opinion yields better predictions of human ratings. We note that with the users' full opinions the correlation is still somewhat low, but the positive trend observed suggests that eliciting more informative statements could lead to a more accurate signal. 

\begin{table*}[tbp]
\centering
\small
\caption{Credit assignment results for \texttt{Llama 3.1 8B Instruction-Tuned}. Darker green indicates larger Z-score. Z-score column is for altered tokens. Alterations are represented by "<misaligned>/<aligned>".}
\label{tab:policy_zscore_comparison}
\newcolumntype{P}[1]{>{\raggedright\arraybackslash}p{#1\textwidth}}
\begin{tabular}{P{0.2} P{0.2} P{0.37} l}
\toprule
\textbf{Reference policy prompt} & \textbf{User policy prompt} & \textbf{Sequence} & \textbf{Z-Score} \\
\midrule
User food profile: empty & User food profile: vegetarian & \tokenhl{2}{I} \tokenhl{9}{ am} \tokenhl{9}{ having} \tokenhl{61}{ chicken/tofu} \tokenhl{10}{ ench}\tokenhl{12}{il}\tokenhl{11}{adas} \tokenhl{1}{ tonight}\tokenhl{8}{.} \tokenhl{5}{ Then} \tokenhl{7}{ I} \tokenhl{70}{ am} \tokenhl{9}{ going} \tokenhl{10}{ to} \tokenhl{7}{ meet} \tokenhl{11}{ up} \tokenhl{12}{ with} \tokenhl{9}{ some} \tokenhl{12}{ friends}\tokenhl{9}{.} & 2.69 \\
\midrule
User location profile: empty &User location profile: lives in a cold climate & \tokenhl{10}{I}
\tokenhl{10}{'m}
\tokenhl{11}{ going}
\tokenhl{11}{ to}
\tokenhl{7}{ the}
\tokenhl{37}{ beach/mountains}
\tokenhl{7}{ this}
\tokenhl{11}{ weekend}
\tokenhl{11}{ to}
\tokenhl{70}{ surf/ski}
\tokenhl{9}{.}
\tokenhl{16}{ I}
\tokenhl{0}{ need}
\tokenhl{11}{ to}
\tokenhl{18}{ buy}
\tokenhl{11}{ some}
\tokenhl{11}{ new}
\tokenhl{3}{ clothes}
\tokenhl{8}{.} & 1.78, 3.31 \\
\midrule
User time profile: empty & User time profile: morning &
\tokenhl{9}{I}
\tokenhl{7}{ am}
\tokenhl{5}{ about}
\tokenhl{8}{ to}
\tokenhl{7}{ eat}
\tokenhl{7}{ some}
\tokenhl{7}{ food}
\tokenhl{8}{.}
\tokenhl{2}{ I}
\tokenhl{2}{ am}
\tokenhl{4}{ going}
\tokenhl{7}{ to}
\tokenhl{0}{ have}
\tokenhl{70}{ spaghetti/pancakes}
\tokenhl{3}{.}
\tokenhl{9}{ I}
\tokenhl{21}{ will}
\tokenhl{7}{ use}
\tokenhl{2}{ my}
\tokenhl{8}{ phone}
\tokenhl{8}{ to}
\tokenhl{8}{ order}
\tokenhl{6}{ it}
\tokenhl{6}{.} & 4.26 \\
\midrule
User opinion: empty & User opinion: Favors stricter gun control laws. &
\tokenhl{16}{Implement}
\tokenhl{14}{ing}
\tokenhl{5}{ background}
\tokenhl{9}{ checks}
\tokenhl{4}{ that}
\tokenhl{14}{ are}
\tokenhl{44}{ less/more}
\tokenhl{14}{ strict}
\tokenhl{17}{ for}
\tokenhl{16}{ gun}
\tokenhl{9}{ purchases}
\tokenhl{15}{ is}
\tokenhl{70}{ essential}
\tokenhl{5}{.}
\tokenhl{10}{ Also}
\tokenhl{14}{,}
\tokenhl{2}{ my}
\tokenhl{8}{ favorite}
\tokenhl{3}{ color}
\tokenhl{17}{ is}
\tokenhl{16}{ orange}
\tokenhl{12}{.}
& 2.09 \\
\bottomrule
\end{tabular}
\end{table*}

\subsection{Evaluating Credit Assignment}
\label{sec:exp_credit_assignment}

\paragraph{Setup.}
Our framework defines token-level rewards as
$r_i(s,a) = \beta \log \pi_i(a \mid s)$,
where each policy $\pi_i$ is obtained by prompting a base LLM with agent-specific information. Reward-guided search is effective when policies exhibit \emph{localized credit assignment}: changes in $\pi_i(a \mid s)$ concentrate on tokens tied to the prompt information. \citet{rafailov2024from} observed this behavior in DPO-trained models; here we test it for policies induced by prompting instruction-tuned models.

\paragraph{Test design.}
We use Llama 3.1 8B Instruction-Tuned. For each test, we compare token log-probabilities under a \textbf{user policy prompt} (e.g., ``User time profile: morning'') and a \textbf{reference policy prompt} (e.g., ``User time profile: empty''). We evaluate two nearly identical sequences: $X_1$ contains a concept that conflicts with the user profile (e.g., ``spaghetti'' for a ``morning'' profile) and $X_2$ replaces it with an aligned concept (e.g., ``pancakes'').

\paragraph{Metric.}
For each token $a_j$ with prefix $s$, we compute the log-likelihood difference between the user and reference policies,
\[
\Delta L(a_j \mid s) = \log \pi_U(a_j \mid s) - \log \pi_R(a_j \mid s).
\]
We then measure the change in this difference when switching from $X_1$ to $X_2$,
\[
D_j = \bigl|\, \Delta L_{X_1}(a_j \mid s) - \Delta L_{X_2}(a_j \mid s) \,\bigr|.
\]
A large $D_j$ indicates that the user profile alters the model’s preference for token $a_j$ specifically when the conflicting concept is swapped for an aligned concept. To compare across positions, we convert $\{D_j\}$ to Z-scores using the mean and standard deviation over all tokens in the sequence.

\paragraph{Findings.}
The largest Z-scores occur at the tokens that differ between $X_1$ and $X_2$ (Table~\ref{tab:credit_assignment}). With a ``morning'' time profile, the ``spaghetti'' versus ``pancakes'' position shows the highest shift. Additional examples with Llama and \texttt{Gemma 2 9B Instruction-Tuned} in \autoref{app:credit_assignment} show the same pattern.

\paragraph{Implication.}
System prompting induces localized credit assignment in instruction-tuned models. This supports using $r_i(s,a) = \beta \log \pi_i(a \mid s)$ from prompted policies as a targeted signal for token-level search.

\subsection{Consensus Generation}
\label{sec:consensus_generation_exp}

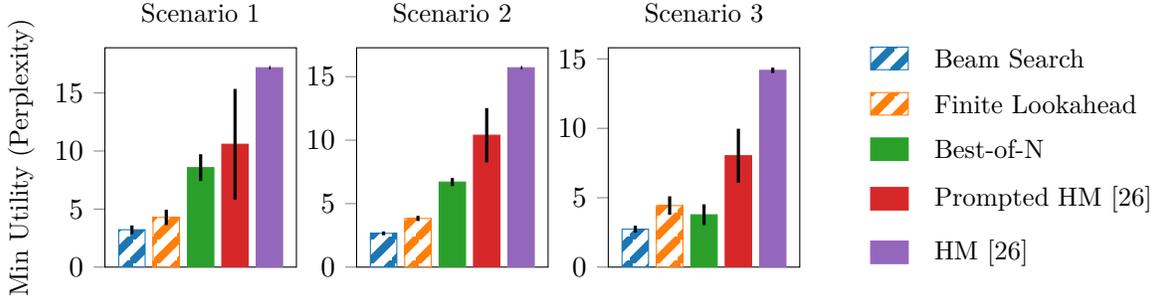
\begin{figure*}[t]
\centering
\begin{tikzpicture}

\pgfdeclarepatternformonly{north east lines wide}%
   {\pgfqpoint{-0pt}{-0pt}}{\pgfqpoint{10pt}{10pt}}{\pgfqpoint{9pt}{9pt}}{
     \pgfsetlinewidth{3pt}
     \pgfpathmoveto{\pgfqpoint{0pt}{0pt}}
     \pgfpathlineto{\pgfqpoint{10pt}{10pt}}
     \pgfsetrectcap
     \pgfusepath{stroke}
   }

\begin{groupplot}[
    group style={group size=3 by 1, horizontal sep=8mm, vertical sep=10mm},
    width=0.25\textwidth,   
    height=4.5cm,
    ymin=0,
    ybar,
    tick align=outside,
    tick pos=left,
    xtick=\empty,
    ylabel style={font=\small},
    title style={font=\small},
    ybar,
    error bars/y dir=both,
    error bars/y explicit,
    error bars/error mark=none,
    enlarge x limits=0.55   
]

\nextgroupplot[title={Scenario 1}, ylabel={Min Utility (Perplexity)}]
\addplot[pattern=north east lines wide, pattern color=tabblue,   draw=tabblue,   bar shift=-26pt]
  coordinates {(1, 3.220311375644632) +- (0, 0.37173748246240773)};
\addplot[pattern=north east lines wide, pattern color=taborange, draw=taborange, bar shift=-13pt]
  coordinates {(1, 4.2852422593667)   +- (0, 0.659336113190132)};
\addplot[fill=tabgreen,  draw=tabgreen,  bar shift=0pt]
  coordinates {(1, 8.569835380968636) +- (0, 1.153808013835064)};
\addplot[fill=tabred,    draw=tabred,    bar shift=13pt]
  coordinates {(1, 10.584603800793127) +- (0, 4.76366408699182)};
\addplot[fill=tabpurple, draw=tabpurple, bar shift=26pt]
  coordinates {(1, 17.175444078474424) +- (0, 0.0)};

\nextgroupplot[title={Scenario 2}]
\addplot[pattern=north east lines wide, pattern color=tabblue,   draw=tabblue,   bar shift=-26pt]
  coordinates {(1, 2.660412049853852) +- (0, 0.14652962793459665)};
\addplot[pattern=north east lines wide, pattern color=taborange, draw=taborange, bar shift=-13pt]
  coordinates {(1, 3.83392221672239)   +- (0, 0.20070975623509832)};
\addplot[fill=tabgreen,  draw=tabgreen,  bar shift=0pt]
  coordinates {(1, 6.698984940067551) +- (0, 0.3183660189792501)};
\addplot[fill=tabred,    draw=tabred,    bar shift=13pt]
  coordinates {(1, 10.387804930225881) +- (0, 2.1369448492699448)};
\addplot[fill=tabpurple, draw=tabpurple, bar shift=26pt]
  coordinates {(1, 15.70550310440665) +- (0, 0.0)};

\nextgroupplot[title={Scenario 3}]
\addplot[pattern=north east lines wide, pattern color=tabblue,   draw=tabblue,   bar shift=-26pt]
  coordinates {(1, 2.7350514770216763) +- (0, 0.25172561558067)};
\addplot[pattern=north east lines wide, pattern color=taborange, draw=taborange, bar shift=-13pt]
  coordinates {(1, 4.434410379363428)  +- (0, 0.6689645736468379)};
\addplot[fill=tabgreen,  draw=tabgreen,  bar shift=0pt]
  coordinates {(1, 3.772227586372919)  +- (0, 0.7508844755621441)};
\addplot[fill=tabred,    draw=tabred,    bar shift=13pt]
  coordinates {(1, 8.02593037282243)   +- (0, 1.9410838204334466)};
\addplot[fill=tabpurple, draw=tabpurple, bar shift=26pt]
  coordinates {(1, 14.175891992587914) +- (0, 0.18954220124329008)};

\end{groupplot}

\path (group c3r1.east) ++(8mm,0) node (legendanchor) {};
\matrix[
  matrix of nodes,
  nodes={anchor=west},
  column sep=6pt,
  row sep=3pt,
  font=\small
] at (legendanchor) [anchor=west] {
  \node[minimum width=5mm, minimum height=3mm,
        pattern=north east lines wide, pattern color=tabblue, draw=tabblue] {}; & \node {Beam Search}; \\
  \node[minimum width=5mm, minimum height=3mm,
        pattern=north east lines wide, pattern color=taborange, draw=taborange] {}; & \node {Finite Lookahead}; \\
  \node[minimum width=5mm, minimum height=3mm, fill=tabgreen,  draw=tabgreen]  {}; & \node {Best-of-N}; \\
  \node[minimum width=5mm, minimum height=3mm, fill=tabred,    draw=tabred]    {}; & \node {Prompted HM~\cite{tessler2024ai}}; \\
  \node[minimum width=5mm, minimum height=3mm, fill=tabpurple, draw=tabpurple] {}; & \node {HM~\citep{tessler2024ai}}; \\
};

\end{tikzpicture}
\caption{Per-scenario egalitarian welfare (perplexity). Lower values indicate better minimum agent utility. Striped bars indicate that the method uses search over the token-level MDP. Numerical results are reported in Table~\ref{tab:eppl_results} in \autoref{app:additional_results}.}
\label{fig:per_scenario_perplexity}
\end{figure*}

We evaluated different approaches for generating a single consensus statement by comparing our proposed search algorithms against several baselines. The primary goal was to assess how well each method optimizes egalitarian welfare (EW), measured by a perplexity-based metric reflecting worst-case agent alignment. More detailed consensus generation experiments, with \texttt{Gemma 2 9b Instruction-Tuned}, an LLM-judge metric, and with more agents are presented in \autoref{app:additional_consensus_experiments} of \autoref{app:additional_results}.

\paragraph{Scenarios:} We used scenarios from the Habermas Machine dataset \citep{tessler2024ai}. To obtain distinct settings, scenario descriptions were embedded using \texttt{BAAI/bge-large-en-v1.5} \citep{bge_embedding} and clustered via $k$-means ($k=3$). Representative scenarios were selected from each cluster (Scenarios 1, 2, and 3). The issues for these scenarios are in the captions of Tables~\ref{tab:democracy_summary}, \ref{tab:UK_summary}, and \ref{tab:vegetarian_summary}.

\paragraph{Agents and Policies:} For each scenario, agent opinions were taken from the dataset. Agent policies $\pi_i$ were instantiated by prompting \texttt{Llama 3.1 8B Instruct}~\citep{grattafiori2024llama} with the issue and agent $i$'s opinion, instructing it to generate text aligned with that viewpoint (prompt shown in \autoref{fig:agent_policy_prompt} in \autoref{app:prompts}). The resulting likelihoods $\pi_i(s,a)$ represent agent $i$'s preferences. Agent opinions are detailed in Tables~\ref{tab:democracy_summary}, \ref{tab:UK_summary}, and \ref{tab:vegetarian_summary} in \autoref{app:sample_statements} of \autoref{app:additional_results}.

\paragraph{Base Generation Model:} Consensus statements were generated using \texttt{Llama 3.1 8B Instruct}, prompted with the issue and all agent opinions (prompt shown in \autoref{fig:ref_policy_prompt} in \autoref{app:prompts}).

\paragraph{Evaluation Metric - Egalitarian Perplexity (EPPL):} To capture alignment with the least satisfied agent for a consensus statement $X$, we define Egalitarian Perplexity. For each agent $i$, their specific perplexity $PPL_i(X)$ is found by prompting \texttt{Llama 3.1 8B Instruct} with the issue and agent $i$'s opinion to generate a statement perfectly reflecting that opinion. The average log-likelihood of the actual consensus statement $X = (a_1, \dots, a_L)$ conditioned on this agent-specific prompt is:
\[
\bar{L}_i(X) = \frac{1}{L} \sum_{t=1}^{L} \log \pi_i(s_{t-1}, a_t | \text{prompt}_i).
\]
The agent-specific perplexity is $PPL_i(X) = \exp(-\bar{L}_i(X))$. The final Egalitarian Perplexity for $X$ is $EPPL(X) = \max_{i \in N} PPL_i(X)$. Lower EPPL indicates better egalitarian welfare.

\paragraph{Seeds:} We report the mean and standard deviation of EPPL over 3 seeds per method and scenario.

\paragraph{Methods Compared:} We compared our proposed algorithms, \textbf{Finite Lookahead} (Algorithm~\ref{alg:finite-lookahead}, depth $d=4$, branching $B=2$), and \textbf{Beam Search} (Algorithm~\ref{alg:beam-search-egalitarian}, width $w=4$, pruning based on partial EPPL), against three baselines: \textbf{Best-of-N} (selecting the best of $N=4$ samples\footnote{$N=4$ was chosen to align with the Prompted Habermas Machine.} from the base model by EPPL); a \textbf{Prompted Habermas Machine}\footnote{As implemented in \url{https://github.com/google-deepmind/habermas_machine}} (1 critique round, 4 candidates\footnote{We chose four candidates to align with the default parameters in the Prompted Habermas Machine example in the Habermas Machine GitHub repository.}, critiques from base model conditioned on agent opinions); and the original \textbf{Habermas Machine} (HM)~\citep{tessler2024ai} consensus (generated by a fine-tuned Chinchilla-70B).

\subsubsection{Results}
\label{sssec:results_consensus_concise}
\autoref{fig:per_scenario_perplexity} summarizess EPPL performance (lower is better). \textbf{Beam Search consistently achieved the lowest EPPL} (average: 2.87), indicating high alignment with the least satisfied agent. \textbf{Finite Lookahead also performed well} (average EPPL: 4.18), outperforming baseline methods. Both search methods surpassed \textbf{Best-of-N} (6.35) and the \textbf{Prompted Habermas Machine} (9.67). The \textbf{Habermas Machine} baseline had the highest EPPL (15.69), possibly because its statement was generated by a different model (Chinchilla 70B).

The results suggest that token-level search guided by EPPL, as in Beam Search and Finite Lookahead, effectively generates consensus statements with better minimum agent alignment compared to sampling or iterative refinement. Consensus statements for the first seed are in Tables~\ref{tab:democracy_summary}, \ref{tab:UK_summary}, and \ref{tab:vegetarian_summary}. The strong empirical performance of methods operating on the token-level MDP complements the fact that these methods are also more amenable to theoretical analysis and fairness guarantees.

\section{Discussion}
\label{sec:discussion}

This work introduced a framework for generating consensus statements by modeling the process as a multi-objective, token-level MDP with rewards derived from agent-specific language model policies. Our aim was to connect LLM-based text generation with the formal fairness guarantees of social choice theory via this MDP.

Our theoretical contributions for stochastic outcomes (lotteries over statements) focused on the core. By maximizing Nash Welfare over expected probability-based utilities, we identified an optimal lottery $p^*$ that induces a stochastic generation policy $\pi^*$ (Definition~\ref{def:induced_policy}) inheriting the ex-ante core property (Corollary~\ref{cor:policy-core}). Chunking was introduced as a heuristic to manage the search space. For deterministic outcomes (single statements), we focused on maximizing egalitarian welfare (EW), proposing finite lookahead and beam search as approximation algorithms.

Empirical results validated several aspects of our framework. We found that token likelihood from prompted policies meaningfully correlate with human preferences. Credit assignment experiments (\autoref{sec:exp_credit_assignment}) confirmed that prompting LLMs with agent profiles creates policies that not only correlate with human preferences, but that also correctly assign credit to tokens. Consensus generation experiments (\autoref{sec:consensus_generation_exp}), using Egalitarian Perplexity (EPPL) to measure EW, showed that beam search and finite lookahead, guided by the EW objective, outperformed baselines like Best-of-N and an adapted Habermas Machine. Beam search yielded the lowest EPPL. 

Overall, formulating consensus generation as a search problem within a token-level MDP, guided by explicit social choice objectives like EW, is a promising direction. However, there are some important questions that future work should address. First, finding methods to train more faithful personalized policies for each agent is an important direction as it is upstream of many important challenges \citep{ijcai2025p1141}. Second, finding a way to approximate the core on the unchunked space without looking at the whole tree is an important step to work towards. Previous work has looked at approximating the core \citep{ebadian2024optimized, gemp2024approximating}, but neither method directly applies to our setting. And third, future work should focus on theoretical guarantees for the single statement case, which we did not obtain. 

Lastly, we note that until our methods are better understood, the outputs of our algorithm should be treated as artifacts for collective sense-making instead of binding decisions, as suggested by \citet{revel2025ai}. For example, instead of treating the output as a decision, the output could be treated as another input to the discussion that participants of the collective decision could reflect on. Optimistically, one would hope that these consensus statements could identify previously unknown points of agreement or solutions that no one had previously thought of that are in fact highly agreeable. In sum, this work contributes theoretical foundations and practical algorithms for incorporating social choice principles into generative AI for collective decision-making and sense-making.

\bibliographystyle{ACM-Reference-Format} 
\bibliography{sample}


\appendix
\onecolumn
\clearpage 

\section*{Table of Contents for Appendices}
\begingroup
\setlength{\parindent}{0pt}
\setlength{\parskip}{0.5ex}

{\textbf{A. Additional Related Work}} \dotfill \pageref{app:additional_related_work} \\
{\textbf{B. Illustration of deriving the policy from the lottery}} \dotfill \pageref{app:illustration} \\
{\textbf{C. Deferred Proofs}} \dotfill \pageref{app:proof} \\
{\textbf{D. Search Algorithms}} \dotfill \pageref{app:search_algos} \\
{\textbf{E. Additional Empirical Results}} \dotfill \pageref{app:additional_results}\\
{\textbf{F. Prompts}} \dotfill \pageref{app:prompts} \\

\endgroup
\clearpage

\section{Additional Related Work}
\label{app:additional_related_work}

\paragraph{Mechanism Design for LLMs.} This nascent area explores mechanisms for settings where multiple agents interact via LLMs. \citet{duetting2024mechanism} design token-level auctions where bids influence generated distributions, analyzing incentive compatibility. \citet{dubey2024auctions} design auctions for incorporating ads into LLM summaries, using ``prominence'' as an intermediate allocation variable. Common ground exists in the high-level goal of aggregating inputs from multiple agents (represented algorithmically/via LLMs) to produce a collective textual output. However, our work differs significantly in methodology and objective. We do not employ economic mechanisms like auctions, bids, or payments. Instead, we formulate the aggregation problem as a multi-objective optimization within an MDP, aiming to achieve a fair consensus based on social choice criteria, rather than allocating influence or generating content based on bids.

\section{Illustration of Deriving the Policy from the Lottery}
\label{app:illustration}

The relationship between the target distribution $p^*$ over leaves and the calculation of the policy $\Pi^*$ at an internal node $s$ is illustrated in Figure~\ref{fig:induced_policy_diagram}.

\begin{figure*}[h]
\centering
\begin{tikzpicture}[
    level distance=2.4cm,
    level 1/.style={sibling distance=4.5cm},
    level 2/.style={sibling distance=2.2cm},
    edge from parent/.style={draw, -latex},
    every node/.style={rectangle, draw, rounded corners, font=\scriptsize, align=center, minimum height=0.6cm},
    leaf/.style={rectangle, draw, font=\tiny, align=center, minimum height=0.5cm},
    prob/.style={font=\tiny, draw=none, midway, sloped, above}
]
\node (s) {$s$} [->]
    child { node (sa) {$s\|k_1$}
        child { node[leaf] (leaf1) {$X_1 \in \mathcal{C}_{\mathcal{K}}(s, k_1)$ \\ $p^*(X_1)=0.2$} edge from parent node[prob] {}}
        child { node[leaf] (leaf2) {$X_2 \in \mathcal{C}_{\mathcal{K}}(s, k_1)$ \\ $p^*(X_2)=0.05$} edge from parent node[prob] {}}
        edge from parent node[left, draw=none] (edge1) {$k_1$}
    }
    child { node (sb) {$s\|k_2$}
        child { node[leaf] (leaf3) {$X_3 \in \mathcal{C}_{\mathcal{K}}(s, k_2)$ \\ $p^*(X_3)=0.1$} edge from parent node[prob] {}}
        child { node[leaf] (leaf4) {$X_4 \in \mathcal{C}_{\mathcal{K}}(s, k_2)$ \\ $p^*(X_4)=0.3$} edge from parent node[prob] {}}
        edge from parent node[right, draw=none] (edge2) {$k_2$}
    };

\node[draw=none, font=\scriptsize, align=left, above =0.1cm of s] (ps) {
    $\begin{aligned}
    P^*(\mathcal{C}_{\mathcal{K}}(s)) &= p^*(X_1)+p^*(X_2) +p^*(X_3)+p^*(X_4) \\
    &= 0.2+0.05+0.1+0.3 \\
    &= 0.65
    \end{aligned}$%
};
\node[draw=none, font=\scriptsize, align=left, below left =0.1cm and 0.5cm of sa] (psk1) {
    $\begin{aligned}
    &P^*(\mathcal{C}_{\mathcal{K}}(s, k_1)) \\
    &= p^*(X_1)+p^*(X_2) \\
    &= 0.2+0.05 = 0.25
    \end{aligned}$%
};
\node[draw=none, font=\scriptsize, align=left, below right=0.1cm and 0.5cm of sb] (psk2) {
    $\begin{aligned}
&P^*(\mathcal{C}_{\mathcal{K}}(s, k_2)) \\
&= p^*(X_3)+p^*(X_4) \\
&= 0.1+0.3 = 0.40
\end{aligned}$
};

\node[draw=none, font=\scriptsize, align=left, above left=-0.5cm and 0.0cm of edge1] (pi1) {
    $\begin{aligned}
    \Pi^*(s, k_1) &= \frac{P^*(\mathcal{C}_{\mathcal{K}}(s, k_1))}{P^*(\mathcal{C}_{\mathcal{K}}(s))} \\
    &= \frac{0.25}{0.65} \approx 0.38
    \end{aligned}$%
};
\node[draw=none, font=\scriptsize, align=left, above right=-0.5cm and -0.4cm of edge2] (pi2) {
    $\begin{aligned}
    \Pi^*(s, k_2) &= \frac{P^*(\mathcal{C}_{\mathcal{K}}(s, k_2))}{P^*(\mathcal{C}_{\mathcal{K}}(s))} \\
    &= \frac{0.40}{0.65} \approx 0.62
    \end{aligned}$%
};

\end{tikzpicture}
\caption{Illustration of the induced stochastic policy $\Pi^*$ at state $s$. The optimal lottery $p^*$ assigns probabilities to the leaf nodes (complete paths). The probability $P^*(\mathcal{C}_{\mathcal{K}}(s))$ is the sum of $p^*(X)$ for all leaves reachable from $s$. The probability $P^*(\mathcal{C}_{\mathcal{K}}(s, k))$ is the sum for leaves reachable via action $k$. The policy $\Pi^*(s, k)$ is the conditional probability of taking action $k$.}
\label{fig:induced_policy_diagram}
\end{figure*}
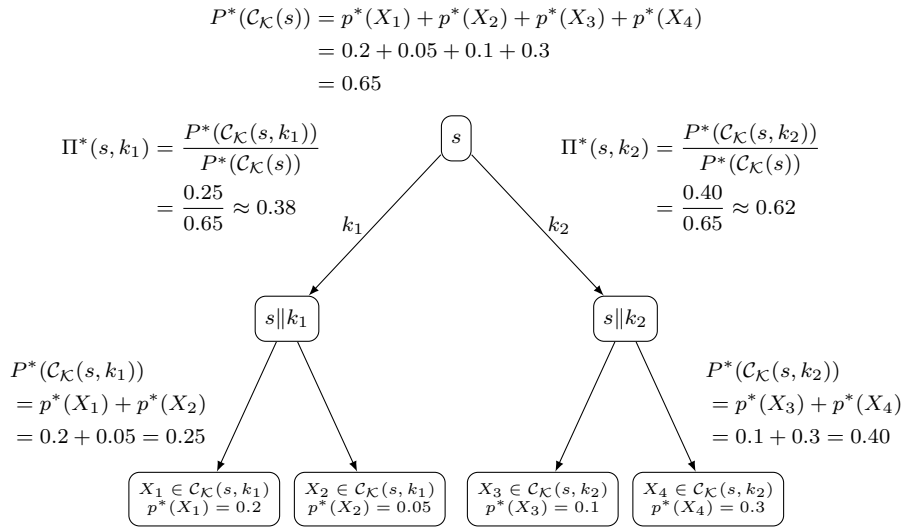

\section{Deferred Proofs}
\subsection{Proof of Theorem 1}
\label{app:proof}
\begin{proof}
We prove by induction on the depth of state $s$ in the chunked tree that the probability of reaching state $s$ under policy $\Pi^*$, denoted $P_{\Pi^*}(s)$, is equal to $P^*(\mathcal{C}_{\mathcal{K}}(s))$, the total mass assigned by $p^*$ to leaves passing through $s$.

\textbf{Base Case (Depth 0):} The initial state is $s_0$. $P_{\Pi^*}(s_0) = 1$ by definition. Also, $\mathcal{C}_{\mathcal{K}}(s_0) = \mathcal{C}_{\mathcal{K}}$ (all paths pass through the start state), and $P^*(\mathcal{C}_{\mathcal{K}}(s_0)) = \sum_{X \in \mathcal{C}_{\mathcal{K}}} p^*(X) = 1$ since $p^*$ is a probability distribution. Thus, $P_{\Pi^*}(s_0) = P^*(\mathcal{C}_{\mathcal{K}}(s_0))$.

\textbf{Inductive Hypothesis (IH):} Assume that for all states $s$ at depth $d$, $P_{\Pi^*}(s) = P^*(\mathcal{C}_{\mathcal{K}}(s))$.

\textbf{Inductive Step:} Consider an arbitrary state $s'$ at depth $d+1$. State $s'$ must be reached from a unique predecessor state $s$ at depth $d$ by taking a specific action (chunk) $k$, such that $s' = s\|k$. The probability of reaching $s'$ under $\Pi^*$ is:
\begin{align*}
P_{\Pi^*}(s') &= P_{\Pi^*}(s) \cdot \Pi^*(s,k) \\
&= P^*(\mathcal{C}_{\mathcal{K}}(s)) \cdot \Pi^*(s,k) \quad \text{(by IH)}
\end{align*}
If $P^*(\mathcal{C}_{\mathcal{K}}(s)) = 0$, then $P_{\Pi^*}(s) = 0$, implying $P_{\Pi^*}(s') = 0$. Also, if $P^*(\mathcal{C}_{\mathcal{K}}(s)) = 0$, then $P^*(\mathcal{C}_{\mathcal{K}}(s, k)) = 0$ since $\mathcal{C}_{\mathcal{K}}(s, k) \subseteq \mathcal{C}_{\mathcal{K}}(s)$. Since $s' = s\|k$, $\mathcal{C}_{\mathcal{K}}(s') = \mathcal{C}_{\mathcal{K}}(s, k)$, so $P^*(\mathcal{C}_{\mathcal{K}}(s')) = 0$. Thus, $P_{\Pi^*}(s') = P^*(\mathcal{C}_{\mathcal{K}}(s')) = 0$.

If $P^*(\mathcal{C}_{\mathcal{K}}(s)) > 0$, we use the definition of $\Pi^*(s,k)$:
\begin{align*}
P_{\Pi^*}(s') &= P^*(\mathcal{C}_{\mathcal{K}}(s)) \cdot \frac{P^*(\mathcal{C}_{\mathcal{K}}(s, k))}{P^*(\mathcal{C}_{\mathcal{K}}(s))} \\
&= P^*(\mathcal{C}_{\mathcal{K}}(s, k))
\end{align*}
Since $s' = s\|k$, we have $\mathcal{C}_{\mathcal{K}}(s') = \mathcal{C}_{\mathcal{K}}(s, k)$. Therefore,
\[
P_{\Pi^*}(s') = P^*(\mathcal{C}_{\mathcal{K}}(s'))
\]
This completes the inductive step.

\textbf{Conclusion:} The induction holds for all states $s$. Now, consider any leaf node $X \in \mathcal{C}_{\mathcal{K}}$. A leaf node is a state at the maximum depth. The set of paths passing through leaf $X$ is just the singleton set $\{X\}$, so $\mathcal{C}_{\mathcal{K}}(X) = \{X\}$. Applying our proven result for state $s=X$:
\[
P_{\Pi^*}(X) = P^*(\mathcal{C}_{\mathcal{K}}(X)) = P^*(\{X\}) = p^*(X)
\]
Since this holds for all $X \in \mathcal{C}_{\mathcal{K}}$, the distribution $p_{\Pi^*}$ induced by policy $\Pi^*$ is identical to the target distribution $p^*$. %
\end{proof}

\subsection{Proof of Theorem 2}

\begin{proof}
Fix $R>1$. Choose parameters $\varepsilon\in(0,1)$ and $\delta\in\bigl(0,(1-\varepsilon)^2/R\bigr)$. Consider two agents $i\in\{1,2\}$ and two leaves $x^\star$ and $x^{-}$. Let $C=\{x^\star,x^{-}\}$. Define utilities
\[
u_i(x^\star)=1-\varepsilon,\qquad u_i(x^{-})=\sqrt{\delta}\quad\text{for }i=1,2.
\]
Assume chunking removes $x^\star$ and keeps $x^{-}$, so $K=\{x^{-}\}$.

First, optimize over $C$. For any lottery $p\in\Delta(C)$ let $\alpha=p(x^\star)\in[0,1]$. Each agent’s expected utility equals
\[
\sum_{x\in C}u_i(x)p(x)=\alpha(1-\varepsilon)+(1-\alpha)\sqrt{\delta}\le 1-\varepsilon,
\]
since $\sqrt{\delta}<1-\varepsilon$. Hence
\[
\mathrm{NW}_C(p)=\Big(\alpha(1-\varepsilon)+(1-\alpha)\sqrt{\delta}\Big)^2\le (1-\varepsilon)^2,
\]
with equality only at $\alpha=1$. Therefore $p_C^{\ast}=\mathbf e_{x^\star}$ and $\mathrm{NW}_C(p_C^{\ast})=(1-\varepsilon)^2$.

Second, optimize over $K$. Any $p\in\Delta(K)$ satisfies
\[
\sum_{x\in K}u_i(x)p(x)\le \sqrt{\delta}\quad\text{for }i=1,2,
\]
so $\mathrm{NW}_K(p)\le\delta$. With $K=\{x^{-}\}$ we have $p_K^* = \mathbf e_{x^{-}}$ and $\mathrm{NW}_K(p_K^*)=\delta$.

Combining the two bounds gives
\[
\frac{\mathrm{NW}_C(p_C^{*})}{\mathrm{NW}_K(p_K^{*})}
\;\ge\; \frac{(1-\varepsilon)^2}{\delta}\;>\;R,
\]
by the choice of $\delta$. This establishes the claim.
\end{proof}

\clearpage
\section{Search Algorithms}
\label{app:search_algos}

\begin{algorithm}[h]
\caption{Finite Lookahead Egalitarian Welfare Maximization}
\label{alg:finite-lookahead}
\begin{algorithmic}[1]
\Require Set of agents $N$, Lookahead depth $d$, Branching factor $B$, Max length $L_{\max}$
\State Initialize current state $s_0$ to the empty sequence; $t \leftarrow 0$
\State Initialize generated path $X_{fl} \leftarrow (s_0)$
\While{$t < L_{\max}$ and $s_t$ is not terminal}
    \State Let $X_{prefix}$ be the path corresponding to $s_t$.
    \State Let $\mathcal{P}_d(s_t)$ be the set of all paths $P=(a_1, \dots, a_k)$ starting from $s_t$ such that $k \le d$ and $X_{prefix} \| P$ does not exceed length $L_{\max}$.
    \State Find a path $P^* = (a^*_1, \dots, a^*_{k^*}) \in \mathcal{P}_d(s_t)$ that maximizes the lookahead objective:
            \[ \max_{P \in \mathcal{P}_d(s_t)} \min_{i \in N} U_i^{\log}(X_{prefix} \| P) \]
    \If{no path $P^*$ found (e.g., $s_t$ is terminal)}
        \State Break
    \EndIf
    \State Take the first action $a^* \leftarrow a^*_1$.
    \State Update state: $s_{t+1} \leftarrow T(s_t, a^*)$
    \State Append $a^*$ to the sequence represented by $X_{fl}$.
    \State $t \leftarrow t + 1$
\EndWhile
\State Perform brush up on $X_{fl}$ (using prompt defined in \autoref{fig:brush_up_prompt}).
\State \Return Complete path $X_{fl}$
\end{algorithmic}
\end{algorithm}

\begin{algorithm}[h]
\caption{Egalitarian Welfare Beam Search}
\label{alg:beam-search-egalitarian}
\begin{algorithmic}[1]
\Require Set of agents $N$, Beam width $w$, Branching factor $B$, Max length $L_{\max}$
\State Initialize beam $\mathcal{B}_0 = \{ (s_0, \text{path } s_0) \}$ with the empty sequence path
\For{$t = 0$ to $L_{\max}-1$}
    \State $\mathcal{C}_{t+1} \leftarrow \emptyset$
    \For{each path $X_{path}$ represented by state sequence $(s_0, \dots, s_t)$ in $\mathcal{B}_t$}
        \If{$s_t$ is not terminal}
            \State Consider $B$ possible next actions $A_B(s_t)$ from state $s_t$
            \For{each action $a \in A_B(s_t)$}
                \State $s_{t+1} \leftarrow T(s_t, a)$
                \State $X_{new\_path} \leftarrow X_{path} \| a$
                \State Add $X_{new\_path}$ (represented by its state sequence) to $\mathcal{C}_{t+1}$
            \EndFor
        \Else
             \State Add $X_{path}$ (already terminal) to $\mathcal{C}_{t+1}$
        \EndIf
    \EndFor
    \State For each path $X \in \mathcal{C}_{t+1}$, compute its score $f(X) = \min_{i \in N} U_i^{\log}(X)$.
    \State $\mathcal{B}_{t+1} \leftarrow$ top $w$ paths from $\mathcal{C}_{t+1}$ according to score $f(X)$.
\EndFor
\State Select path from final beam $\mathcal{B}_{L_{\max}}$ with the highest score $f(X)$.
\State Perform brush up on selected path (using prompt defined in \autoref{fig:brush_up_prompt}).
\State \Return Brushed up statement.
\end{algorithmic}
\end{algorithm}

\clearpage
\section{Additional Empirical Results}
\label{app:additional_results}

\subsection{Credit Assignment}
\label{app:credit_assignment}

\begin{table*}[htbp]
\centering
\small
\caption{Credit assignment results for \texttt{Gemma 2 9b Instruction-Tuned}. Darker green indicates larger Z-score. Z-score column is for altered tokens. Alterations are represented by "<misaligned>/<aligned>".}
\label{tab:gemma_policy_zscore_comparison}
\newcolumntype{P}[1]{>{\raggedright\arraybackslash}p{#1\textwidth}}
\begin{tabular}{P{0.23} P{0.23} P{0.34} l}
\toprule
\textbf{Reference policy prompt} & \textbf{User policy prompt} & \textbf{Sequence} & \textbf{Z-Score} \\
\midrule
User food profile: empty & User food profile: vegetarian &
\tokenhl{0}{I}
\tokenhl{0}{ am}
\tokenhl{0}{ having}
\tokenhl{70}{ chicken/tofu}
\tokenhl{6}{ en}
\tokenhl{11}{chila}
\tokenhl{11}{das}
\tokenhl{11}{ tonight}
\tokenhl{0}{.}
\tokenhl{8}{ Then}
\tokenhl{5}{ I}
\tokenhl{0}{ am}
\tokenhl{2}{ going}
\tokenhl{9}{ to}
\tokenhl{4}{ meet}
\tokenhl{10}{ up}
\tokenhl{11}{ with}
\tokenhl{9}{ some}
\tokenhl{11}{ friends}
\tokenhl{3}{.}
& 3.69 \\
\midrule
User location profile: empty &User location profile: lives in a cold climate & \tokenhl{0}{I}
\tokenhl{0}{'}
\tokenhl{0}{m}
\tokenhl{0}{ going}
\tokenhl{0}{ to}
\tokenhl{0}{ the}
\tokenhl{24}{ beach/mountains}
\tokenhl{2}{ this}
\tokenhl{17}{ weekend}
\tokenhl{13}{ to}
\tokenhl{70}{ surf/ski}
\tokenhl{16}{.}
\tokenhl{15}{ I}
\tokenhl{4}{ need}
\tokenhl{16}{ to}
\tokenhl{24}{ buy}
\tokenhl{17}{ some}
\tokenhl{16}{ new}
\tokenhl{10}{ clothes}
\tokenhl{11}{.} & 1.01, 2.86 \\
\midrule
User time profile: empty & User time profile: morning & 
\tokenhl{0}{I}
\tokenhl{0}{ am}
\tokenhl{0}{ about}
\tokenhl{0}{ to}
\tokenhl{0}{ eat}
\tokenhl{0}{ some}
\tokenhl{0}{ food}
\tokenhl{0}{.}
\tokenhl{0}{ I}
\tokenhl{0}{ am}
\tokenhl{0}{ going}
\tokenhl{0}{ to}
\tokenhl{0}{ have}
\tokenhl{70}{ spaghetti/pancakes}
\tokenhl{3}{.}
\tokenhl{9}{ I}
\tokenhl{7}{ will}
\tokenhl{7}{ use}
\tokenhl{5}{ my}
\tokenhl{1}{ phone}
\tokenhl{9}{ to}
\tokenhl{9}{ order}
\tokenhl{7}{ it}
\tokenhl{9}{.} & 3.13 \\
\midrule
User opinion: empty & User opinion: Favors stricter gun control laws. & 
\tokenhl{23}{Implementing}
\tokenhl{23}{ background}
\tokenhl{23}{ checks}
\tokenhl{23}{ that}
\tokenhl{23}{ are}
\tokenhl{56}{ less/more}
\tokenhl{70}{ strict}
\tokenhl{7}{ for}
\tokenhl{56}{ gun}
\tokenhl{6}{ purchases}
\tokenhl{2}{ is}
\tokenhl{21}{ essential}
\tokenhl{0}{.}
\tokenhl{6}{ Also}
\tokenhl{23}{,}
\tokenhl{7}{ my}
\tokenhl{12}{ favorite}
\tokenhl{69}{ color}
\tokenhl{23}{ is}
\tokenhl{8}{ orange}
\tokenhl{23}{.}
& 1.77 \\ 
\bottomrule
\end{tabular}
\end{table*}

\subsection{Additional Consensus Generation Experiments}
\label{app:additional_consensus_experiments}
We ran additional experiments on consensus generation. We restricted to scenarios with fewer than five agent opinions. We then applied \(k\)-means clustering (\(k=5\)) to scenario embeddings from \texttt{BAAI/bge-large-en-v1.5}. Consensus statements were generated with \texttt{Gemma 2 9B (instruction-tuned)}. We evaluated these statements using Egalitarian Perplexity (\(EPPL\); lower is better), computed both in-model (the generating model, \texttt{Gemma 2 9B Instruct}) and out-of-model (\texttt{Llama 3.1 8B Instruct}). We also used an LLM-judge metric with GPT-4.1 (prompt in \autoref{fig:llm_judge_prompt}). For each scenario, GPT-4.1 ranked candidates on behalf of each agent; we report the maximum (worst) rank across agents for each method (lower is better). All averages are over three random seeds.

\subsubsection{Question 1: Scaling—Habermas vs.\ Best-of-\(N\)}
\label{app:scaling_q1}
We vary the number of candidates for Habermas and \(N\) for Best-of-\(N\). \autoref{tab:gemma_eppl_scaling_comparison_q1_detail} and \autoref{tab:llama_eppl_scaling_comparison_q1_detail} report means and standard deviations of \(EPPL\) (lower is better).

\begin{table}[htbp]
\centering
\footnotesize
\begin{minipage}{.48\textwidth}
\centering
\caption{Gemma \(EPPL\) (lower is better): Habermas vs.\ Best-of-\(N\) scaling.}
\label{tab:gemma_eppl_scaling_comparison_q1_detail}
\begin{tabular}{@{}llrr@{}}
\toprule
Method & \(N\) & Mean & Std.\ Dev. \\
\midrule
Habermas Machine & 1 & 15.54 & 6.16 \\
Habermas Machine & 2 & 17.98 & 6.68 \\
Habermas Machine & 3 & 19.14 & 7.96 \\
Habermas Machine & 5 & 17.11 & 7.06 \\
Habermas Machine & 10 & 17.38 & 6.19 \\
Habermas Machine & 20 & 19.61 & 8.39 \\
Habermas Machine & 50 & 15.55 & 4.51 \\
\midrule
Best-of-\(N\) & 1 & 18.59 & 11.35 \\
Best-of-\(N\) & 3 & 11.99 & 5.99 \\
Best-of-\(N\) & 5 & 9.75 & 5.29 \\
Best-of-\(N\) & 10 & 9.03 & 4.03 \\
Best-of-\(N\) & 20 & 6.92 & 2.27 \\
Best-of-\(N\) & 50 & 7.43 & 2.91 \\
\bottomrule
\end{tabular}
\end{minipage}\hfill
\begin{minipage}{.48\textwidth}
\centering
\caption{Llama \(EPPL\) (lower is better): Habermas vs.\ Best-of-\(N\) scaling.}
\label{tab:llama_eppl_scaling_comparison_q1_detail}
\begin{tabular}{@{}llrr@{}}
\toprule
Method & \(N\) & Mean & Std.\ Dev. \\
\midrule
Habermas Machine & 1 & 12.20 & 7.39 \\
Habermas Machine & 2 & 12.05 & 2.87 \\
Habermas Machine & 3 & 13.69 & 2.38 \\
Habermas Machine & 5 & 12.48 & 4.80 \\
Habermas Machine & 10 & 11.73 & 2.56 \\
Habermas Machine & 20 & 13.37 & 2.29 \\
Habermas Machine & 50 & 12.82 & 3.81 \\
\midrule
Best-of-\(N\) & 1 & 11.92 & 4.08 \\
Best-of-\(N\) & 3 & 9.14 & 3.33 \\
Best-of-\(N\) & 5 & 7.67 & 2.43 \\
Best-of-\(N\) & 10 & 7.67 & 2.43 \\
Best-of-\(N\) & 20 & 7.32 & 2.12 \\
Best-of-\(N\) & 50 & 7.58 & 2.58 \\
\bottomrule
\end{tabular}
\end{minipage}
\end{table}

\paragraph{Findings for RQ1.}
Across all \(N\), Best-of-\(N\) yields lower \(EPPL\) than Habermas in both evaluations. Best-of-\(N\) improves as \(N\) increases, with marginal gains after \(N=20\). Habermas shows no clear improvement with more candidates. Under a fixed sample budget, Best-of-\(N\) reduces \(EPPL\) more effectively.

\subsubsection{Question 2: Beam Search Scaling}
\label{app:scaling_q2}
We examine how beam width affects performance.

\begin{table}[htbp]
\centering
\footnotesize
\caption{Beam width scaling (average over 5 scenarios, 3 seeds each). Lower is better for all metrics.}
\label{tab:beam_search_scaling_q2}
\begin{tabular}{@{}rrrrrrr@{}}
\toprule
Beam & \multicolumn{2}{c}{Gemma \(EPPL\)} & \multicolumn{2}{c}{Llama \(EPPL\)} & \multicolumn{2}{c}{LLM Judge Max Rank} \\
Width & Mean & Std.\ Dev. & Mean & Std.\ Dev. & Mean & Std.\ Dev. \\
\midrule
2 & 10.01 & 1.95 & 10.07 & 1.70 & 3.13 & 0.80 \\
4 & 7.84 & 1.33 & 10.16 & 1.99 & 2.80 & 0.61 \\
6 & 10.54 & 10.40 & 9.20 & 2.51 & 3.13 & 0.38 \\
8 & 12.56 & 7.47 & 12.69 & 5.90 & 3.93 & 0.15 \\
\bottomrule
\end{tabular}
\end{table}

\paragraph{Findings for RQ2.}
The optimal width depends on the metric. Width \(4\) is best on average for Gemma \(EPPL\) and the LLM-judge rank; width \(6\) is best for Llama \(EPPL\). Width \(8\) degrades performance, likely due to myopic pruning based on partial \(EPPL\) scores within the beam search.

\subsubsection{Question 3: Method Comparison}
\label{app:scaling_q3}
We compare Finite Lookahead (\(d=3\), \(B=3\)), Beam Search (width \(=4\)), Best-of-\(N\) (\(N\in\{5,10,50\}\)), and Habermas (candidates \(=5\)).

\begin{table}[htbp]
\centering
\footnotesize
\caption{Overall comparison: LLM-judge max rank (lower is better).}
\label{tab:method_comp_llm_judge_q3}
\begin{tabular}{@{}lrrrr@{}}
\toprule
Method & Max Rank & Std.\ Dev. & Min Max Rank & Max Max Rank \\
\midrule
Best-of-\(N\) (\(N=50\)) & 3.33 & 0.72 & 2.00 & 4.00 \\
Best-of-\(N\) (\(N=10\)) & 4.33 & 0.72 & 3.00 & 5.00 \\
Beam Search (width \(=4\)) & 5.07 & 0.80 & 4.00 & 6.00 \\
Best-of-\(N\) (\(N=5\)) & 5.33 & 0.72 & 4.00 & 6.00 \\
Habermas (candidates \(=5\)) & 5.40 & 0.74 & 4.00 & 6.00 \\
Finite Lookahead (depth \(=3\)) & 5.80 & 0.41 & 5.00 & 6.00 \\
\bottomrule
\end{tabular}
\end{table}

\begin{table}[htbp]
\centering
\footnotesize
\caption{Overall comparison: Gemma \(EPPL\) (lower is better).}
\label{tab:method_comp_gemma_eppl_q3}
\begin{tabular}{@{}lrrrr@{}}
\toprule
Method & Mean & Std.\ Dev. & Min & Max \\
\midrule
Best-of-\(N\) (\(N=50\)) & 7.43 & 2.91 & 3.92 & 10.96 \\
Beam Search (width \(=4\)) & 7.84 & 1.33 & 6.34 & 9.89 \\
Finite Lookahead (depth \(=3\)) & 8.01 & 2.23 & 5.60 & 10.97 \\
Best-of-\(N\) (\(N=10\)) & 9.03 & 4.03 & 3.33 & 14.52 \\
Best-of-\(N\) (\(N=5\)) & 9.75 & 5.29 & 4.72 & 18.42 \\
Habermas (candidates \(=5\)) & 17.11 & 7.06 & 9.44 & 32.49 \\
\bottomrule
\end{tabular}
\end{table}

\begin{table}[htbp]
\centering
\footnotesize
\caption{Overall comparison: Llama \(EPPL\) (lower is better).}
\label{tab:method_comp_llama_eppl_q3}
\begin{tabular}{@{}lrrrr@{}}
\toprule
Method & Mean & Std.\ Dev. & Min & Max \\
\midrule
Best-of-\(N\) (\(N=50\)) & 7.58 & 2.58 & 3.85 & 10.79 \\
Best-of-\(N\) (\(N=5\)) & 7.67 & 2.43 & 4.76 & 10.95 \\
Best-of-\(N\) (\(N=10\)) & 7.67 & 2.43 & 3.69 & 10.17 \\
Finite Lookahead (depth \(=3\)) & 8.49 & 3.29 & 5.56 & 13.89 \\
Beam Search (width \(=4\)) & 10.16 & 1.99 & 7.96 & 13.06 \\
Habermas (candidates \(=5\)) & 12.48 & 4.80 & 6.48 & 23.30 \\
\bottomrule
\end{tabular}
\end{table}

\paragraph{Findings for RQ3.}
Best-of-\(N\) with large \(N\) (\(N=50\)) is strongest on both \(EPPL\) metrics and on worst-case LLM-judge rank. Beam Search with width \(4\) is competitive on Gemma \(EPPL\) (second to Best-of-\(N\) with \(N=50\)) and outperforms Best-of-\(N\) with \(N\in\{5,10\}\) on that metric. Finite Lookahead performs moderately well. Habermas (candidates \(=5\)) lags on both \(EPPL\) and the judge metric. Method choice trades off average \(EPPL\), worst-case judge rank, and compute budget. Token-level MDP methods (Beam Search and Finite Lookahead) also admit analysis within the MDP framework, which is useful for theory.

\subsubsection{Bowling~Green\slash Warren~County (41 agents; Llama~3.1 8B Instruct)}
\label{app:bgwc_llama41}
We evaluated the Bowling~Green\slash Warren~County data from \citet{boehmer2025generative}. Using \(41\) agents and one seed with Llama~3.1 8B Instruct, \autoref{tab:bgwc_llama41} reports Egalitarian Perplexity (\(EPPL\) calculated using Llama~3.1 8B Instruct; lower is better) and the full consensus statement for each method.

\begin{table}[htbp]
\centering
\footnotesize
\caption{Bowling~Green\slash Warren~County results (41 agents, one seed). Model: Llama~3.1 8B Instruct. Lower \(EPPL\) is better.}
\label{tab:bgwc_llama41}
\begin{tabularx}{\textwidth}{@{}l r X@{}}
\toprule
Method & \(EPPL\downarrow\) & Full consensus statement \\
\midrule
Beam Search (width \(=4\)) & 3.29 & To make Bowling~Green/Warren~County a better place to live, work, and spend time, we need to address our community's most pressing issues, including a lack of accessible and affordable housing options, inadequate public transportation, and insufficient mental-health resources. Furthermore, we need to focus on improving our community's infrastructure, such as upgrading our roads and traffic-management systems. Additionally, we must ensure our community is inclusive and welcoming to all residents, regardless of their background, culture, or socioeconomic status. \\
Best-of-4 & 6.17 & To make Bowling~Green/Warren~County a better place to live, work, and spend time, we recommend prioritizing initiatives that expand high-speed internet access, improve traffic flow and safety, enhance public transportation, and enrich community life, while also prioritizing education, public health, and economic equality. \\
Prompted Habermas Machine (4 candidates \(+\) 1 critique) & 14.08 & To make Bowling~Green/Warren~County a better place to live, work, and spend time, we need to address pressing issues such as increasing high-speed internet access, improving traffic flow and safety by installing traffic cameras, increasing arts-education opportunities, and protecting tenant and renters' rights. We also need to ensure welcoming and inclusive initiatives to support refugees and immigrants in broader community development. Invest in public transportation and enforce safety regulations such as parking policies to foster a sustainable environment. By working together to implement these actionable plans and initiatives, we can build a more comprehensive strategy to address the community's high-priority needs. \\
\bottomrule
\end{tabularx}
\end{table}

\subsection{Supplemental Table for \autoref{fig:per_scenario_perplexity}}
\label{app:supp_table}

\begin{table}[H]
\centering
\footnotesize
\caption{Egalitarian Welfare (perplexity) across all scenarios shown in \autoref{fig:per_scenario_perplexity}.}
\label{tab:eppl_results}
\begin{tabular}{l|ccc|c}
\toprule
\textbf{Method} & \textbf{Scenario 1} & \textbf{Scenario 2} & \textbf{Scenario 3} & \textbf{Overall} \\
\midrule
Beam Search & \textbf{3.22 ± 0.37} & \textbf{2.66 ± 0.15} & \textbf{2.74 ± 0.25} & \textbf{2.87 ± 0.37} \\
Finite Lookahead & 4.29 ± 0.66 & 3.83 ± 0.20 & 4.43 ± 0.67 & 4.18 ± 0.61 \\
Best-of-N & 8.57 ± 1.15 & 6.70 ± 0.32 & 3.77 ± 0.75 & 6.35 ± 2.14 \\
Habermas Machine & 10.58 ± 4.76 & 10.39 ± 2.14 & 8.03 ± 1.94 & 9.67 ± 3.42 \\
Predefined & 17.18 ± 0.00 & 15.71 ± 0.00 & 14.18 ± 0.19 & 15.69 ± 1.23 \\
\bottomrule
\end{tabular}
\end{table}

\clearpage
\subsection{Scenarios and Sample Consensus Statements from \autoref{sec:consensus_generation_exp}}
\label{app:sample_statements}

{%
  \setlength{\textfloatsep}{6pt}   
  \setlength{\abovecaptionskip}{3pt}
  \renewcommand{\topfraction}{0.99}
  \renewcommand{\textfraction}{0.01}
\begin{table*}[!b]
    \centering
    \scriptsize
    \caption{Opinions and Consensus Statements on Scenario 1: "\textit{Are we living in a golden age of democracy?}"}
    \label{tab:democracy_summary}
    \begin{tabular}{@{} l p{0.7\textwidth} @{}}
        \toprule
        \textbf{Source} & \textbf{Opinion / Consensus Statement} \\
        \midrule
        \multicolumn{2}{c}{\textit{\textbf{Individual Agent Opinions}}} \\
        \midrule
        Agent 1 & No, I think the golden age of democracy is long gone. I think a system where the first past the post wins is not working and we need to move to a model of proportional representation which would give more people the feeling that their voices were being heard. On the subject of voices, I'm strongly of the opinion that we have become a society where the loud voices of minority groups are able to impose their views on the rest of the population which to me is no democracy at all. \\
        \midrule
        Agent 2 & Worldwide democracy is more present than it's ever been in history. So yes, compared to previous ages in history I believe we are. Although that's not to say we can't improve - many countries still don't operate democratically, and in the ones that do, corruption is rife. \\
        \midrule
        Agent 3 & Yes, we are living in a golden age of democracy as democracy is of the people. \\
        \midrule
        Agent 4 & Not at all. The notion of democracy is being used for personal gains of those in government, and the system is manipulated. Around the world there is a considerable amount of oppression and lack of democractic values. \\
        \midrule
        Agent 5 & Compared to some parts of the world such as Russia and China which are actively regressing and reverting back to archaic ways of controlling their people, most Western countries are living through comparatively decent times, although problems still exist. \\
        \midrule
        \multicolumn{2}{c}{\textit{\textbf{Consensus Statements from Methods}}} \\
        \midrule
        Best of N & While acknowledging the imperfect nature of democracies worldwide, and recognising instances of regression and manipulation, we contend that, compared to historical periods and several non-democratic countries, the majority of Western countries and a sizeable portion of the global population currently enjoys a relatively high level of democratic freedoms; however, the overarching consensus leans towards a need for more inclusive and accountable forms of governance. \\
        \midrule
        Finite Lookahead & Given the varying perspectives, we acknowledge that the definition of democracy can be subjective and dependent on geographical and cultural contexts, yet still, we recognize the existence of significant flaws in the current democratic systems worldwide, including instances of corruption, oppression, and manipulation, leading to the erosion of trust in governments and the democratic process. While some countries have made significant progress in consolidating democratic institutions, others have regressed, highlighting the need for continuous improvement and reform to ensure the preservation and strengthening of democratic values and principles. Therefore, we conclude that the current state of democracy worldwide is complex and multifaceted, with areas of improvement necessary to fully realize the ideals of democracy for all citizens, and acknowledge the presence of democratic systems in many parts of the world, particularly in the Western hemisphere, which, despite their imperfections, remain the best available option for promoting the common good and protecting individual rights and freedoms. \\
        \midrule
        Prompted Habermas\\Machine & We acknowledge that the concept of a golden age of democracy is complex and contentious. While democracy has made significant progress in many parts of the world, we also recognize that fundamental challenges persist, including issues with representation, corruption, and the suppression of marginalized voices. These challenges highlight the need for a more critical examination of our democratic systems and a commitment to ongoing reform and improvement. \\
        \midrule
        Habermas Machine & Although in the past we've had high hopes of a better world after the horrors of WWII and the fall of the Iron Curtain, democracy is in danger worldwide and may never reach its full potential. The Western world has poor democratic values, and even though democracy is spreading worldwide it is being overshadowed by the loud voices of minority groups. \\
        \midrule
        Beam Search & We are not living in a golden age of democracy, but rather in an age where democracy is being challenged from within and from outside, and it is up to all of us to ensure that the principles of democracy are upheld. \\
        \bottomrule
    \end{tabular}
\end{table*}
}

\begin{table*}[htbp]
    \centering
    \scriptsize
    \caption{Agent Opinions and Method Consensus on Scenario 2: "\textit{Should the UK's ties to Europe be stronger or weaker?}"}
    \label{tab:UK_summary}
    \begin{tabular}{@{} l p{0.7\textwidth} @{}}
        \toprule
        \textbf{Source} & \textbf{Opinion / Consensus Statement} \\
        \midrule
        \multicolumn{2}{c}{\textit{\textbf{Individual Agent Opinions}}} \\
        \midrule
        Agent 1 & When we was in Europe we had good trade with them , The decision to leave was very bad for united kindom . We need to put the vote again to the British public i am sure this time the decision would be to remain \\
        \midrule
        Agent 2 & The natural evolution of our species has been to grow into ever bigger "tribes". Families ruled by their patriarchs became tribes ruled by elders became countries ruled by governments. It made sense that countries would evolve separately since they were geographically separate with no means of communication. Now our world is so connected, it is inevitable that we evolve into ever larger units such as the United States and the European Union. Eventually we will become a multi-planetary species ruled by an Earth government. To sever ties with Europe is a step in the wrong direction. \\
        \midrule
        Agent 3 & Although we did exit EU few years ago, we are geographically and economically in proximity to most EU countries. Several geographic, financial, political and economical parameters are intertwined with our neighbouring countries and, it would be advantageous to be in good relations to fully harness our economic, political, and financial facilities. \\
        \midrule
        Agent 4 & I believe the UK's ties to Europe should be stronger, as it would make trade deals easier to negotiate. This would allow us to benefit from a larger array of goods, which would make our imports cheaper. I also believe that closer ties with Europe in terms of immigration policies should occur as our immigration policy doesn't coincide with other nations around us. I believe a more united Europe would help all countries grow more through the movement of free labour, goods and services and more. \\
        \midrule
        Agent 5 & I think they should be stronger, as all the countries of Europe, except us, are part of the EU, and it makes Britain more isolated. Countries are stronger I believe when they work together with other countries and are open to the influence of other countries. I believe that leaving the EU was a step backwards and done for the wrong reasons. We are not the country we were 100 or even 50 years ago and the world is changing so much that we cannot afford to not be part of Europe. \\
        \midrule
        \multicolumn{2}{c}{\textit{\textbf{Consensus Statements from Methods}}} \\
        \midrule
        Best of N & Given the participants' opinions, we conclude that in order to achieve economic prosperity and streamline international relations, the UK's ties to Europe should be strengthened through collaborative efforts in trade, immigration, and economic policies while embracing the benefits of European integration, mutual understanding, and cooperation. \\
        \midrule
        Finite Lookahead & The UK's ties to Europe should be strengthened to promote economic growth, facilitate trade, and foster greater unity among European nations. This can be achieved through mutually beneficial agreements and policies that reflect the shared interests of the UK and its European neighbours, leading to a more harmonious and prosperous region for all.  Let's work together to build a stronger, more integrated Europe for the benefit of all its citizens, including the British people.  Let's put the past behind us and look forward to a brighter, more collaborative future for the UK and Europe.  Let's recognize the benefits of unity and work towards creating a more cohesive and successful Europe for generations to come. \\
        \midrule
        Prompted Habermas\\Machine & We propose strengthening ties with Europe to foster greater economic cooperation, promote a more unified Europe, and address global challenges. However, we acknowledge that the path to this goal should include a reevaluation of the decision to leave the EU, allowing for a more informed and inclusive discussion about our relationship with Europe. Moreover, we recognize the long-term implications of European integration as part of a broader process of global unification, which may ultimately lead to a more interconnected and cooperative world. \\
        \midrule
        Habermas Machine & The UK's ties to Europe should be stronger. This is because, although the UK did leave the EU, we are geographically and economically in proximity to most EU countries. Several geographic, financial, political and economical parameters are intertwined with our neighbouring countries and, it would be advantageous to be in good relations to fully harness our economic, political, and financial facilities. \\
        \midrule
        Beam Search & The UK's ties to Europe should be stronger, and it would be beneficial for the country to be a part of the European Union or a similar entity that fosters cooperation and unity among its member states. This would allow for more efficient trade and economic growth, as well as greater collaboration on issues such as immigration, security, and environmental policy. \\
        \bottomrule
    \end{tabular}
\end{table*}

\begin{table*}[htbp]
    \centering
    \scriptsize
    \caption{Agent Opinions and Method Consensus on Scenario 3: "\textit{Is it okay to use animals for food?}"}
    \label{tab:vegetarian_summary}
    \begin{tabular}{@{} l p{0.7\textwidth} @{}}
        \toprule
        \textbf{Source} & \textbf{Opinion / Consensus Statement} \\
        \midrule
        \multicolumn{2}{c}{\textit{\textbf{Individual Agent Opinions}}} \\
        \midrule
        Agent 1 & Yes, as human beings are natural meat eaters. However, it is morally wrong to keep animals in cruel conditions and use them for food. On the whole, farming is an unethical practise and many farming methods should be banned.  \\
        \midrule
        Agent 2 & I think that it is ok to use animals for food because I enjoy eating meat. I don't think it is ok if the animal is endangered and I do think that there should be certain factors in place to ensure that they are well kept.  \\
        \midrule
        Agent 3 & In my opinion it is ok to use animals for food. They provide nutrition and vitamins to use and I think it is just the basic food chain. I do understand why people choose to be vegatarian or vegan, but I do not agree with them trying to force their views upon others. \\
        \midrule
        Agent 4 & Humans have used animals for food since we first evolved as a species. Animals provide a great deal of nutrition that the human body needs to be healthy. It is certainly true that we now have more alternative foods available so that it can be possible to live healthily without consuming meat - though this does usually rely on taking nutritional supplements. I think that we should continue to diversify food production to provide people with meat alternatives, but also we should continue to improve the welfare of animals that are eaten for food and try to do this sustainably to minimise the impact this has on the animals wellbeing and on the land / environment too. \\
        \midrule
        \multicolumn{2}{c}{\textit{\textbf{Consensus Statements from Methods}}} \\
        \midrule
        Best of N & While considering the ethical and nutritional aspects of using animals for food, we agree that it is acceptable to consume meat in moderation, provided that animals are raised and slaughtered humanely, and that sustainable and environmentally-friendly farming methods are implemented to minimize harm. \\
        \midrule
        Finite Lookahead & The use of animals for food is acceptable, however it should be done in a way that respects the welfare of the animals and considers the impact on the environment, with the aim of achieving sustainable food production methods that benefit both human and animal well-being, and also providing alternatives for those who choose not to consume meat.  This includes improving farming practices, reducing the amount of waste and pollution, and promoting education on responsible food production and consumption practices.  All sides of the issue should be taken into account, and the goal should be to achieve a balance between food security, animal welfare, and environmental sustainability. \\
        \midrule
        Prompted Habermas\\Machine & We acknowledge that humans have traditionally used animals for food, and we recognize the importance of improving animal welfare and minimizing environmental impact to ensure sustainable practices. We understand that some consumers prefer eating meat, and we encourage diversifying food production options to include alternatives that can meet the nutritional needs of humans. Importantly, we emphasize the need to prohibit inhumane farming practices and actively work towards improving animal conditions, which aligns with our collective commitment to animal well-being and environmental stewardship. \\
        \midrule
        Habermas Machine & Humans have used animals for food since we first evolved as a species. Animals provide a great deal of nutrition that the human body needs to be healthy. It is certainly true that we now have more alternative foods available so that it can be possible to live healthily without consuming meat - though this does usually rely on taking nutritional supplements. We should continue to diversify food production to provide people with meat alternatives, but also we should continue to improve the welfare of animals that are eaten for food and try to do this sustainably to minimise the impact this has on the animals wellbeing and on the land / environment too. \\
        \midrule
        Beam Search & It is generally acceptable to use animals for food, as long as they are treated with respect and care, but we must also consider the impact of our food choices on the environment and animal welfare. \\
        \bottomrule
    \end{tabular}
\end{table*}

\clearpage

\section{Prompts}
\label{app:prompts}

\begin{figure}[H]
\begin{lstlisting}
    
You are generating a consensus statement that represents the views of multiple participants.

Your task is to continue the statement in a way that addresses the issue and considers all participants' opinions. Be concise and keep the statement short (less than 50 tokens) and focused. ONLY WRITE THE STATEMENT AND NOTHING ELSE.

Issue:
<issue>

Participants' opinions:
<opinion_1>
.
.
.
<opinion_n>

Consensus statement:
\end{lstlisting}
\caption{Reference policy prompt.}
\label{fig:ref_policy_prompt}
\end{figure}

\begin{figure}[H]
    \begin{lstlisting}
You are generating a statement that represents the views of a single participant.

Your task is to continue the statement in a way that addresses the issue and considers ONLY this participant's opinion. Be concise and keep the statement short (less than 50 tokens) and focused. ONLY WRITE THE STATEMENT AND NOTHING ELSE.

Issue:
<Issue>

Participant's opinion:
<opinion>

Statement reflecting ONLY this participant's opinion:
    \end{lstlisting}
    \caption{Agent policy prompt.}
    \label{fig:agent_policy_prompt}
\end{figure}

\begin{figure}[H]
    \begin{lstlisting}

You are helping to fix ONLY the ending of a generated statement.

VERY IMPORTANT: If the statement ending is already complete and well-formed, DO NOT modify it at all.

Your task is to:
1. DO NOT change any part of the statement except the last few sentences if they have issues
2. Look for and fix ONLY these issues at the end of the statement:
   - Remove repetition in the final sentences
   - Complete any unfinished final sentence that can be completed easily
   - Remove any incomplete final sentence that cannot be meaningfully finished
3. Keep the changes minimal and focused only on the ending
4. DO NOT add any new information or opinions
5. DO NOT modify anything except problematic sentences at the end
6. If the statement is already complete and well-formed, return it EXACTLY as provided

Here is the statement:

<statement>

    \end{lstlisting}
    \caption{Brush up prompt.}
    \label{fig:brush_up_prompt}
\end{figure}

\begin{figure}[H]
\begin{lstlisting}
You are evaluating consensus statements from the perspective of a specific agent. Your task is to rank multiple statements based on how well they represent the agent's opinion and interests on a given issue. Use ONLY the agent's stated opinion to determine the ranking.

Issue:
<issue>

Agent <agent_id>'s Opinion:
<agent_opinion>

Statements to Rank
Statement 1:
<statement_1_text>

Statement 2:
<statement_2_text>
...
Statement <n_statements>:
<statement_n_text>

Task:
From Agent <agent_id>'s perspective, rank all statements from most favorable (1)
to least favorable (<n_statements>) based on how well they represent the agent's
opinion and interests.

Provide your ranking as a JSON object with:
1. 'reasoning': brief explanation for your ranking decisions
2. 'ranking': an array of statement numbers in ranked order (best to worst)

For example: {'reasoning': 'Statement 3 best represents the agent's concerns about...',
'ranking': [3, 1, 2]}
\end{lstlisting}
\caption{Prompt used for LLM judge in \autoref{app:additional_consensus_experiments}.}
\label{fig:llm_judge_prompt}
\end{figure}

\end{document}